\documentclass[a4paper]{article}

\usepackage{amsmath}
\usepackage{amssymb}
\usepackage{amsthm}
\usepackage[english]{babel}
\usepackage{color}
\usepackage{enumerate}
\usepackage{graphicx}
\usepackage{hyperref}
\usepackage{relsize} 
\usepackage{subcaption} 
\usepackage{xspace}
\usepackage{natbib} 
\usepackage{verbatim}

\renewcommand{\cite}[1]{\citep{#1}} 

\numberwithin{table}{section}
\numberwithin{figure}{section}
\numberwithin{equation}{section}

\newcommand{\refitem}[1]{(\ref{#1})}

\theoremstyle{definition}
\newtheorem{theorem}{Theorem}[section]
\newtheorem{example}[theorem]{Example}

\newtheorem{definition}[theorem]{Definition}
\newtheorem{proposition}[theorem]{Proposition}
\newtheorem{remark}[theorem]{Remark}

\newcommand{\fname}[1]{\mathit{#1}} 
\newcommand{\romI}{\emph{(i)}} 
\newcommand{\romII}{\emph{(ii)}} 
\newcommand{\romIII}{\emph{(iii)}} 
\renewcommand{\qed}{\hfill$\square$} 
\newcommand{\cs}{,\,} 
\newcommand{\proofpara}[1]{\bigskip\noindent\textbf{#1.} }
\newcommand{\proofsubpara}[1]{\medskip\quad\textbf{#1.} }
\newcommand{\java}[1]{\texttt{#1}}

\newcommand{\nat}{\mathbb{N}}
\newcommand{\natzero}{\mathbb{N}_0} 
\newcommand{\real}{\mathbb{R}}
\newcommand{\set}[1]{\{#1\}}
\newcommand{\ssize}[1]{\left|#1\right|}

\newcommand{\actions}{A}
\newcommand{\act}{a}
\newcommand{\actB}{b}
\newcommand{\stateset}{V} 

\newcommand{\rl}[1]{L_{#1}} 
\newcommand{\reduceX}[2]{\fname{reduce}(#1,#2)}
\newcommand{\reduce}{\fname{reduce}(\task)}

\newcommand{\task}{T}
\newcommand{\tasktup}{(\states,\initstates,\actions,\rewards,\tr)}
\newcommand{\states}{Q}
\newcommand{\initstates}{\states_0}
\newcommand{\rewards}{\mathit{rewards}}
\newcommand{\tr}{\delta}
\newcommand{\st}{q}
\newcommand{\stB}{r}
\newcommand{\goals}[1]{\fname{goals}(#1)}

\newcommand{\learnableF}{learnable under fairness}
\newcommand{\learnable}{learnable}

\newcommand{\actX}[1]{\act_{#1}}
\newcommand{\cnf}{c}
\newcommand{\cnfX}[1]{\cnf_{#1}}
\newcommand{\stX}[1]{\st_{#1}}
\newcommand{\wmX}[1]{w_{#1}}
\newcommand{\polX}[1]{\pi_{#1}}
\newcommand{\wm}{\wmX{}}
\newcommand{\pol}{\polX{}}
\newcommand{\cnftupX}[1]{(\stX{#1},\polX{#1},\wmX{#1})}
\newcommand{\cnftup}{\cnftupX{}}
\newcommand{\trans}{t}
\newcommand{\opt}[1]{\fname{opt}(#1)}
\newcommand{\apply}[3]{\fname{apply}(#1,#2,#3)}
\newcommand{\jump}[1]{{}\xrightarrow{#1}}

\newcommand{\run}{\mathcal{R}}
\newcommand{\trial}{\chain}
\newcommand{\chain}{\mathcal{C}}

\newcommand{\ground}{\mathit{ground}(\pol)}
\newcommand{\forward}{\fname{forward}(\pol)}
\newcommand{\backward}{\fname{backward}(\pol)}
\newcommand{\fl}[1]{F_{#1}} 
\newcommand{\bl}[1]{B_{#1}} 

\newcommand{\unstable}{\fname{unstable}(\task)}
\newcommand{\border}{\fname{border}(\task)}

\newcommand{\offsets}[1]{\fname{offset}(#1)}
\newcommand{\move}[2]{\fname{move}(#1, #2)}
\newcommand{\gridgoals}{\mathit{Goals}}
\newcommand{\gridactions}{\actions_{\mathrm{grid}}}

\newcommand{\order}[1]{\fname{order}(#1)}
\newcommand{\len}[1]{|#1|}
\newcommand{\elemat}[2]{#1[#2]}
\newcommand{\qnt}[2]{Q(#1,#2)}
\newcommand{\floor}[1]{\left\lfloor #1 \right\rfloor}

\begin{document}

\title{On the convergence of cycle detection for navigational reinforcement learning
}

\author{%
    Tom~J.~Ameloot%
        \thanks{T.J.~Ameloot is a Postdoctoral Fellow of the Research Foundation -- Flanders (FWO).        
        }~   
    \and
    Jan~Van~den~Bussche
}

\date{}

\maketitle{}

\begin{abstract}    
    We consider a reinforcement learning framework where agents have to navigate from start states to goal states.
    We prove convergence of a cycle-detection learning algorithm on a class of tasks that we call reducible.
    Reducible tasks have an acyclic solution.
    We also syntactically characterize the form of the final policy. This characterization can be used to precisely detect the convergence point in a simulation.
    Our result demonstrates that even simple algorithms can be successful in learning a large class of nontrivial tasks.
    In addition, our framework is elementary in the sense that we only use basic concepts to formally prove convergence.
\end{abstract}


\section{Introduction}

Reinforcement Learning (RL) is the subfield of Artificial Intelligence concerned with agents that have to learn a task-solving policy by exploring state-action pairs and observing rewards \cite{sutton-barto_1998}. Off-policy algorithms such as Q-learning, or on-policy algorithms such as Sarsa, are well-understood and can be shown to converge towards optimal policies under quite general assumptions. These algorithms do so by updating, for every state-action pair $(s,a)$, an estimate $Q(s,a)$ of the expected value of doing $a$ in $s$.

Our aim in this article is expressly \emph{not} to propose a more efficient or more powerful new RL algorithm. In contrast, we want to show that convergence can occur already with very simplistic algorithms. The setting of our result is that of tasks where the agent has to reach a goal state in which a reward action can be performed. Actions can be nondeterministic. We like to refer to this setting as \emph{navigational learning}.

The learning algorithm we consider is for a simplistic agent that can only remember the states it has already visited. The algorithm is on-policy; its only update rule is that, when a state is revisited, the policy is revised and updated with an arbitrary new action for that state. 
We refer to this algorithm as the \emph{cycle-detection algorithm}.
Our main result is that this algorithm converges for all tasks that we call \emph{reducible}. Intuitively, a task is reducible if there exists a policy that is guaranteed to lead to reward. 
We also provide a test for convergence that an outside observer could apply to decide when convergence has happened, which can be used to detect convergence in a simulation.
We note that the final policy is allowed to explore only a strict subset of the entire state space.

\medskip

A first motivation for this work is to understand how biological organisms can be successful in learning navigational tasks.
For example, animals can learn to navigate from their nest to foraging areas and back again~\cite{geva-sagiv_2015}. Reward could be related to finding food or returning home. As in standard RL, the learning process might initially exhibit exploration, after which eventually a policy is found that leads the animal more reliably to reward.
In the context of biologically plausible learning, \citet{fremaux_2013} make the following interesting observations.
First, navigational learning is not restricted to physical worlds, but can also be applied to more abstract state spaces. 
Second, the formed policy strongly depends on the experiences of the agent, and therefore the policy is not necessarily optimal.
We elaborate these observations in our formal framework.
We consider a general definition of tasks, which can be used to represent both physically-inspired tasks and more abstract tasks. 
Furthermore, we do not insist on finding (optimal) policies that generate the shortest path to reward, but we are satisfied with learning policies that avoid cycles.

A secondary motivation for this work is to contribute towards filling an apparent gap that exists between the field of Reinforcement Learning and the more logic-based fields of AI and computer science. 
Indeed, on the structural level, the notion of task as used in RL is very similar to the notion of \emph{interpretation} in description logics~\cite{dlhandbook_2010}, or the notion of \emph{transition system} used in verification~\cite{baier-katoen_2008}. Yet, the methods used in RL to establish convergence are largely based on techniques from numerical mathematics and the theory of optimization. 
Our aim was to give proofs of convergence that are more elementary and are more in the discrete-mathematics style common in the above logic-based fields, as well as in traditional correctness proofs of algorithms \cite{algorithms_2009}.

Standard RL convergence proofs assume the condition that state-action pairs are visited (and thus updated) infinitely often, see e.g.~\cite{watkins_1992}. Conditions of this kind are known as \emph{fairness} conditions in the theory of concurrent processes \cite{francez_1986}. Also for our convergence proof we need an appropriate fairness assumption to the effect that when the agent repeats some policy-updating configuration infinitely often, it must also explore all possible updates infinitely often.

We note that the cycle-detection learning algorithm could be remotely related to biologically plausible mechanisms. In some models of biological learning~\cite{potjans_2011,fremaux_2013}, a policy is represented by synaptic connections from neurons encoding (the perception of) states to neurons encoding actions. Connections are strengthened when pre-before-post synaptic activity is combined with reward~\cite{schultz_2013}, causing an organism to remember action preferences for encountered states. If an organism would initially have a policy that frequently leads to cycles in the task, there is a (slow) way to still unlearn that policy, as follows.%
    \footnote{In this discussion, we purposely do not mention mechanisms of disappointment, i.e., the opposite of reward, because the framework in this article does not contain such mechanisms.}
 Consider a pair $(\st,\act)$ of a state $\st$ and its preferred action $\act$ in the policy. Due to noise~\cite{maass_noise_2014}, a neuron $x$ participating in the encoding of action $\act$ could become activated just before state $\st$ effectively occurs.
Possibly, this post-before-pre synaptic activity leads to long-term-depression~\cite{gerstner_book_2014}, i.e., connections from $\st$ to $x$ are weakened.%
    \footnote{    
    If neuron $x$ is activated by noise just before state $\st$ occurs, refractoriness could  prevent state $\st$ from subsequently activating $x$~\cite{gerstner_book_2014}. 
    The resulting absence of a postsynaptic spike at $x$ fails to elicit long-term-potentiation, i.e., connections from $\st$ to $x$ are not strengthened. So, the mentioned weakening effect is not compensated.}
At some synapses, the weakening effect is aided by a longer time window for long-term-depression compared to long-term-potentiation~\cite{markram_2011}.
So, if reward would remain absent for longer periods, as in cycles without reward, noise could gradually unlearn action preferences for states. In absence of such preferences, noise could generate random actions for states. The unlearning phase followed by new random action proposals, would resemble our cycle-detection algorithm.

\paragraph{Outline}
This article is organized as follows. 
We discuss related work in Section~\ref{sec:relwork}.
We formalize important concepts in Section~\ref{sec:fundamentals}.
We present and prove our results in Section~\ref{sec:results}. 
We discuss examples and simulations in Section~\ref{sec:examples}, and we conclude in Section~\ref{sec:conclusion}.

\section{Related Work}
\label{sec:relwork}

Some previous work on reinforcement learning algorithms is focused on learning a policy efficiently, say, using a polynomial number of steps in terms of certain input parameters of the task~\cite{kearns_2002,brafman_2002,strehl_2009}.
There is also a line of work in reinforcement learning that is not necessarily aimed towards efficiently bounding the learning time. In that case, convergence of the learning process happens in the limit, by visiting task states infinitely often.
Some notable examples are temporal-difference learning~\cite{sutton_1988,sutton-barto_1998} and Q-learning~\cite{watkins_1989,watkins_1992}. 
Temporal-difference learning has become an attractive foundation for biological learning models~\cite{potjans_2011,fremaux_2013,schultz_2013,schultz_2015}.

Most previous works in numerical reinforcement learning try to find optimal policies, and their related optimal value functions~\cite{sutton_1988,watkins_1992,dayan_1992,dayan_1994,jaakkola_1994,tsitsiklis_1994}; an optimal policy gives the highest reward in the long run.
This has motivated the design of numerical learning techniques.
The corresponding proof techniques do not always clearly illuminate how properties of the task state space interplay with a particular learning algorithm. With the framework introduced in this article, we hope to shed more light on properties of the task state space, in particular on the way that paths could be formed in the graph structure of the task. 
Although our graph-oriented framework has a different viewpoint compared to standard numerical reinforcement learning, we believe that our Theorem~\ref{theo:algvisit}, showing that convergence always occurs on reducible tasks, and its proof contribute to making the fascinating idea of reinforcement learning more easily accessible to a wider audience.
Our convergence result has a similar intent as previous results showing that numerical learning algorithms converge with probability one.

Various papers study models of reinforcement learning in the context of neuronal agents that learn to navigate a physical environment~\cite{vasilaki_2009,potjans_2011,fremaux_2013}.
Interestingly, \citet{fremaux_2013} study both physical and more abstract state spaces. 
As an example of a physical state space, they consider a navigation task in which a simulated mouse has to swim to a hidden platform where it can rest, where resting corresponds to reward; each state contains only the $x$ and $y$ coordinate.
As an example of an abstract state space, they consider an acrobatic swinging task where reward is given when the tip of a double pendulum reaches a certain height; this space is abstract because each state contains two angles and two angular velocities, i.e., there are four dimensions.
Conceptually it does not matter how many dimensions a state space has, because the agent is always just seeking paths in the graph structure of the task.

This idea of finding paths in the task state space is also explored by \citet{bonet_2006}, in a framework based on depth-first search. Their framework has a more global perspective where learning operations have access to multiple states simultaneously and where the overall search is strongly embedded in a recursive algorithm with backtracking.
Our algorithm acts from the local perspective of a single agent, where only one state can be observed at any time. 

As remarked by \citet[p.~104]{sutton-barto_1998}, a repeated theme in reinforcement learning is to update the policy (and value estimation) while the agent visits states. This theme is also strongly present in the current article, because for each visited state the policy always remembers the lastly tried action for that state. The final aim for convergence, as studied in this article, is to eventually not choose any new actions anymore for the encountered states.

The notion of reducibility discussed in this article is related to the principles of (numerical) dynamic programming, upon which a large part of reinforcement learning literature is based~\cite{sutton-barto_1998}. Indeed, in reducibility, we defer the responsibility of obtaining reward from a given state to one of the successor states under a chosen action. This resembles the way in dynamic programming that reward prediction values for a given state can be estimated by looking at the reward prediction values of the successor states.
In settings of standard numerical reinforcement learning, dynamic programming finds an optimal policy in a time that is worst-case polynomial in the number of states and actions. This time complexity is also applicable to our iterative reducibility procedure given in Section~\ref{sub:task}.

\section{Navigational Reinforcement Learning}
\label{sec:fundamentals}

We formalize tasks and the notion of reducibility in Section~\ref{sub:task}. 
Next, in Section~\ref{sub:navlearn}, we use an operational semantics to formalize the interaction between a task and our cycle-detection learning algorithm.
In Section~\ref{sub:convergence}, we define convergence as the eventual stability of the policy.
Lastly, in Section~\ref{sub:fairness}, we impose certain fairness restrictions on the operational semantics.

\subsection{Tasks and Reducibility}
\label{sub:task}
\paragraph{Tasks}

To formalize tasks, we use nondeterministic transition systems where some transitions are labeled as being immediately rewarding, where reward is only an on-off flag.
Formally, a \emph{task} is a five-tuple 
\[
    \task=\tasktup
\]
where 
    $\states$, $\initstates$, and $\actions$ are nonempty finite sets;
    $\initstates\subseteq\states$;
    $\rewards$ is a nonempty subset of $\states\times\actions$; and,
    $\tr$ is a function that maps each $(\st,\act)\in\states\times\actions$ to a nonempty subset of $\states$.
    The elements of $\states$, $\initstates$, and $\actions$ are called respectively \emph{states}, \emph{start states}, and \emph{actions}.
    The set $\rewards$ tells us which pairs of states and actions give immediate reward.
    Function $\tr$ describes the possible successor states of applying actions to states.
    
\begin{remark}
    Our formalization of tasks keeps only the graph structure of models previously studied in reinforcement learning; essentially, compared to finite Markov decision processes~\cite{sutton-barto_1998}, we omit transition probabilities and we simplify the numerical reward signals to boolean flags.
    We do not yet study negative feedback signals, so performed actions give either reward or no reward, i.e., the feedback is either positive or neutral. 
    %
    In our framework, the agent can observe states in an exact manner, which is a commonly used assumption~\cite{sutton-barto_1998,kearns_2002,brafman_2002,bonet_2006}.
    We mention negative feedback signals and partial information as topics for further work in Section~\ref{sec:conclusion}.
    \qed
\end{remark}

\paragraph{Reducibility}
Let $\task$ be a task as above.
We define the set
\[
    \goals\task=\set{\st\in\states\mid \exists\act\in\actions\text{ with }(\st,\act)\in\rewards}.
\]
We refer to the elements of $\goals\task$ as \emph{goal states}.
Intuitively, for a goal state there is an action that reliably gives immediate reward. 
Each task has at least one goal state because the set $\rewards$ is always nonempty.
The agent could learn a strategy to reduce all encountered states to goal states, and then perform a rewarding action at goal states. This intuition is formalized next.

Let $\stateset\subseteq\states$. We formalize how states can be reduced to $\stateset$.
Let $\natzero$ denote the set of natural numbers without zero.
First, we define the infinite sequence 
\[
    \rl 1, \rl 2, \ldots
\]
of sets where $\rl 1=\stateset$, and for each $i\geq 2$,
\[
    \rl i = \rl{i-1}\cup \set{\st\in\states \mid \exists\act\in\actions\text{ with }\tr(\st,\act)\subseteq\rl{i-1} }.
\]
We call $\rl 1$, $\rl 2$, etc, the \emph{(reducibility) layers}.
We define $\reduceX\task\stateset=\bigcup_{i\in\natzero}\rl i$.
Note that $\reduceX\task\stateset\subseteq\states$.
Because $\states$ is finite, there is an index $n\in\natzero$ for which $\rl n=\rl{n+1}$, i.e., $\rl n$ is a \emph{fixpoint}.
Letting $\stateset'\subseteq\states$, we say that $\stateset'$ is \emph{reducible to $\stateset$} if $\stateset'\subseteq\reduceX\task\stateset$.
Intuitively, each state in $\stateset'$ can choose an action to come closer to $\stateset$. 
We also say that a single state $\st\in\states$ is \emph{reducible to $\stateset$} if $\st\in\reduceX\task\stateset$.

Now, we say that task $\task$ is \emph{reducible (to reward)} if the state set $\states$ is reducible to $\goals\task$.
We use the abbreviation $\reduce=\reduceX\task\stateset$ where $\stateset=\goals\task$.
Reducibility formalizes a sense of solvability of tasks.

We illustrate the notion of reducibility with the following example.
\begin{example}
    \label{ex:reduce}
    We consider the task $\task=\tasktup$ defined as follows:
        $\states=\set{1, 2, 3}$;
        $\initstates=\set{1}$;
        $\actions=\set{a, b}$;
        $\rewards=\set{(3, a), (3, b)}$; and, regarding $\tr$, we define
        \begin{align*}
            \tr(1,a) &= \set{1, 3},\\
            \tr(1,b) &= \set{2},\\
            \tr(2,a) &= \set{1, 3},\\
            \tr(2,b) &= \set{3},\\
            \tr(3,a)=\tr(3,b)&=\set{3}.
        \end{align*}    
    Task $\task$ is visualized in Figure~\ref{fig:reduce}.
    Note that the task is reducible, by assigning the action $\actB$ to both state $1$ and state $2$. The reducibility layers up to and including the fixpoint, are:
    \begin{align*}
        \rl 1 &= \goals\task = \set{3},\\ 
        \rl 2 &= \set{3,2}, \\
        \rl 3 &= \set{3,2,1}.
    \end{align*}
        
    For simplicity, the assignments $1\mapsto b$ and $2\mapsto b$ form a deterministic strategy to reward. But we could easily extend task $\task$ to a task in which the strategy to reward is always subjected to nondeterminism, by adding a new state $4$ with the new mappings $\tr(1,b)=\set{2, 4}$, $\tr(4,a)=\tr(4,b)=\set{3}$.
    \qed
\end{example}

\begin{figure}
    \begin{center}
    \includegraphics[width=0.5\textwidth]{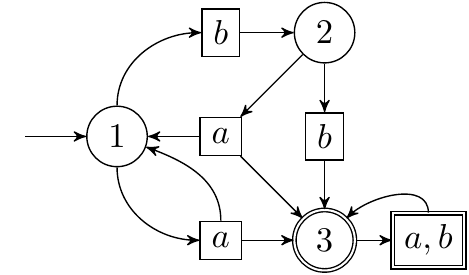}
    \end{center}
    \caption{The task from Example~\ref{ex:reduce}. %
    States and action applications are represented by circles and boxes respectively. Start states are indicated by an arrow without origin. Goal states and their rewarding actions are highlighted by a double circle and double box, respectively.}
    \label{fig:reduce}
\end{figure}

\begin{remark}
    Reducibility formalizes the intuition of an acyclic solution. This appears to be a natural notion of solvability, even in state graphs that contain cycles~\cite{bonet_2006}.
    
    We would like to emphasize that reducibility is a notion of progress in the task transition graph, but it is not the same as determinism because each action application, i.e., transition, remains inherently nondeterministic. 
    We may think of reducibility as onion layers in the state space: the core of the onion consists of the goal states, where immediate reward may be obtained, and, for states in outer layers there is an action that leads one step down to an inner layer, closer to reward. When traveling from an outer layer to an inner layer, the nondeterminism manifests itself as unpredictability on the exact state that is reached in the inner layer.
    \qed
\end{remark}

\subsection{Cycle-detection Algorithm}
\label{sub:navlearn}
We describe a cycle-detection learning algorithm that operates on tasks, by means of an operational semantics that describes the steps taken over time.
We first give the intuition behind the cycle-detection algorithm, and then we proceed with the formal semantics.

\subsubsection{Intuition}
We want to formally elaborate the intuition of path learning. Our aim therefore is not necessarily to design another efficient learning algorithm.
It seems informative to seek only the bare ingredients necessary for navigational learning.
How would such a simple algorithm look like?

As a first candidate, let us consider the algorithm that is given some random initial policy and that always follows the policy during execution.  
There would not be any exploration, and no learning, since the policy is always followed and never modified. In general, the policy might not even lead to any reward at all, and the agent might run around in cycles without obtaining reward.

At the opposite end of the spectrum, there could be a completely random process, that upon each visit to a task state always chooses some random action. 
If the agent is lucky then the random movement through the state space might occasionally, but unreliably, lead to reward.
There is no sign of learning here either, because there is no storage of previously gained knowledge about where reward can be obtained.

Now we consider the following in-between strategy: the algorithm could only choose random actions when it detects a cycle in the state space before reaching reward.
If the agent does not escape from the cycle then it might keep running around indefinitely without ever reaching reward.
More concretely, we could consider a cycle-detection algorithm, constituted by the following directives: 
\begin{itemize}
    \item Starting from a given start state, we continuously remember all encountered states. Each time when reward is obtained, we again forget about which states we have seen.
    \item Whenever we encounter a state $\st$ that we have already seen before, we perform some random action, and we store that action in the policy (for that state $\st$).
\end{itemize}
The cycle-detection algorithm is arguably amongst the simplest learning algorithms that one could conceive. 
The algorithm might be able to gradually refine the policy to avoid cycles, causing the agent to eventually follow an acceptable path to reward.
The working memory is a set containing all states that are visited before obtaining reward. The working memory is reset whenever reward is obtained.

\subsubsection{Operational Semantics}
\label{sub:opsem}
We now formalize the cycle-detection algorithm.
In the following, let $\task=\tasktup$ be a task.

\paragraph{Configurations}
A \emph{configuration} of $\task$ is a triple $\cnf=\cnftup$, where 
    $\st\in\states$;
    $\pol$ maps each $\st\in\states$ to an element of $\actions$; and
    $\wm\subseteq\states$.
The function $\pol$ is called the \emph{policy}.
The set $\wm$ is called the \emph{working memory} and it contains the states that are already visited during the execution, but we will reset $\wm$ to $\emptyset$ whenever reward is obtained.
We refer to $\st$ as the \emph{current state} in the configuration, and we also say that $\cnf$ \emph{contains} the state $\st$.
Note that there are only a finite number of possible configurations.
The aim of the learning algorithm is to refine the policy during trials, as we formalize below.

\paragraph{Transitions}
We formalize how to go from one configuration to another, to represent the steps of the running algorithm.
Let $\cnf=\cnftup$ be a configuration.
We say that $\cnf$ is \emph{branching} if $\st\in\wm$; this means that configuration $\cnf$ represents a revisit to state $\st$, and that we want to generate a new action for the current state $\st$.
Next, we define the set $\opt\cnf$ as follows: 
    letting $\actions'=\actions$ if $\cnf$ is branching and $\actions'=\set{\pol(\st)}$ otherwise,
    we define 
    \[
    \opt\cnf=\set{(\act,\st')\mid \act\in\actions'\text{ and }\st'\in\tr(\st,\act)}.
    \]
Intuitively, $\opt\cnf$ contains the \emph{options} of actions and successor states that may be chosen directly after $\cnf$. 
If $\cnf$ is branching then all actions may be chosen, and otherwise we must restrict attention to the action stored in the policy for the current state. 
Note that the successor state depends on the chosen action.

Next, for a configuration $\cnfX 1=\cnftupX 1$ and a pair $(\act,\st')\in\opt{\cnfX 1}$, we define the \emph{successor configuration} $\cnfX 2=\cnftupX 2$ that results from the application of $(\act,\st')$ to $\cnfX 1$, as follows: 
\begin{itemize}   
    \item 
        $\stX 2=\st'$;
    \item
        $\polX 2(\stX 1)=\act$ and $\polX 2(\stB)=\polX 1(\stB)$ for all $\stB\in\states\setminus\set{\stX 1}$; and,
    \item 
        $\wmX 2=\wmX 1 \cup \set{\stX 1}$.
\end{itemize}
We emphasize that only the action and visited-status of the state $\stX 1$ is modified, where $\stX 1$ is the state that is departed from.
We denote the successor configuration as $\apply{\cnfX 1}\act{\st'}$.

A \emph{transition} $\trans$ is a four-tuple $(\cnfX 1,\act,\st',\cnfX 2)$, also denoted as $\cnfX 1\jump{\act\cs\st'}\cnfX 2$, where $\cnfX 1=\cnftupX 1$ and $\cnfX 2=\cnftupX 2$ are two configurations, $(\act\cs\st')\in\opt{\cnfX 1}$, and $\cnfX 2=\apply{\cnfX 1}\act{\st'}$.
We refer to $\cnfX 1$ and $\cnfX 2$ as the \emph{source configuration} and \emph{target configuration}, respectively.
We say that $\trans$ is a \emph{reward transition} if $(\stX 1,\act)\in\rewards$.
Note that there are only a finite number of possible transitions because there are only a finite number of possible configurations.

\paragraph{Trials and Runs}
A \emph{chain} $\chain$ is a sequence of transitions where for each pair $(\trans_1,\trans_2)$ of subsequent transitions, the target configuration of $\trans_1$ is the source configuration of $\trans_2$. Chains could be finite or infinite.

A \emph{trial} is a chain $\chain$ where either 
    \romI\ the chain is infinite and contains no reward transitions; or,
    \romII\ the chain is finite, ends with a reward transition, and contains no other reward transitions.
To rephrase, if a trial is finite then it ends at the first occurrence of reward; and, if there is no reward transition than the trial must be infinite.

In a trial, we say that an occurrence of a configuration is \emph{terminal} if that occurrence is the last configuration of the trial, i.e., the occurrence is the target configuration of the only reward transition. Note that an infinite trial contains no terminal configurations.
    
A \emph{start configuration} is any configuration $\cnf=\cnftup$ where $\st\in\initstates$ and $\wm=\emptyset$; no constraints are imposed on the policy $\pol$.

Now, a \emph{run} $\run$ on the task $\task$ is a sequence of trials, where 
\begin{enumerate}
    \item the run is either an infinite sequence of finite trials, or the run consists of a finite prefix of finite trials followed by one infinite trial;     
    \item the first configuration of the first trial is a start configuration; 
    \item \label{enu:run-trial-succession} whenever one (finite) trial ends with a configuration $\cnfX 1=\cnftupX 1$ and the next trial starts with a configuration $\cnfX 2=\cnftupX 2$, we have 
        \romI\ $\stX 2\in\initstates$; 
        \romII\ $\polX 2=\polX 1$; and,
        \romIII\ $\wmX 2=\emptyset$;%
            \footnote{Note that $\cnfX 2$ satisfies the definition of start configuration.}
        and,
    \item \label{enu:run-trial-starts} if the run contains infinitely many trials then each state $\st\in\initstates$ is used at the beginning of infinitely many trials.
\end{enumerate}
We put condition~\refitem{enu:run-trial-succession} in words: when one trial ends, we start the next trial with a start state, we reuse the policy, and we again reset the working memory.
By resetting the working memory, we forget which states were visited before obtaining the reward. 
The policy is the essential product of a trial.
Condition~\refitem{enu:run-trial-starts}, saying that each start state is used at the beginning of infinitely many trials, expresses that we want to learn the whole task, with all possible start states.

To refer to a precise occurrence of a trial in a run, we use the ordinal of that occurrence, which is a nonzero natural number.

\begin{remark}
    In the above operational semantics, the agent repeatedly navigates from start states to goal states.
    After obtaining immediate reward at a goal state, the agent's location is always reset to a start state.
    One may call such a framework episodic~\cite{sutton-barto_1998}.
    We note that our framework can also be used to study more continuing operational processes, that do not always enforce a strong reset mechanism from goal states back to remote start states. Indeed, a task could define the set of start states simply as the set of all states. In that case, there are runs possible where some trials start at the last state reached by the previous trial, as if the agent is trying to obtain a sequence of rewards; but we still reset the working memory each time when we begin a new trial.
    \qed
\end{remark}

\subsection{Convergence}
\label{sub:convergence}

We now define a convergence property to formalize when learning has stopped in a run.
Consider a task $\task=\tasktup$.
Let $\run$ be a run on $\task$.

\begin{definition}
We say that a state $\st\in\states$ \emph{(eventually) becomes stable} in $\run$ if 
    there are only finitely many non-terminal occurrences of branching configurations containing $\st$.
\end{definition}

An equivalent definition is to say that after a while there are no more branching configurations at non-terminal positions containing $\st$. 
Intuitively, eventual stability of $\st$ means that after a while there is no risk anymore that $\st$ is paired with new actions, so $\st$ will definitely stay connected to the same action.%
    \footnote{If a branching configuration $\cnf$ is terminal in a trial, $\cnf$ can not influence the action of its contained state because there is no subsequent transition anymore.}
Note that states appearing only a finite number of times in $\run$ always become stable under this definition.

We say that the run $\run$ \emph{converges} if \romI\ all trials terminate (with reward), and \romII\ eventually all states become stable.
We say that the task $\task$ is \emph{\learnable} if all runs on $\task$ converge.

\begin{remark}
    In a run that converges, note that the policy will eventually become fixed because the only way to change the policy is through branching configurations at non-terminal positions. 
    The lastly formed policy in a run is called the final policy, which is studied in more detail in Section~\ref{sub:final}.
    We emphasize that a converging run never stops, because runs are defined as being infinite; the final policy remains in use indefinitely, but it is not updated anymore.
    
    We would also like to emphasize that in a converging run, eventually, the trials contain no cycles before reaching reward: the only moment in a trial where a state could be revisited, is in the terminal configuration, i.e., in the target configuration of the reward transition.
    \qed
\end{remark}

\subsection{Fairness}
\label{sub:fairness}

There are two choice points in each transition of the operational semantics: 
\begin{itemize}
    \item if the source configuration of the transition is branching, i.e., the current state is revisited, then we choose a new random action for the current state; and, 
    \item whenever we apply an action $\act$ to a state $\st$, we can in general choose among several possible successor states in $\tr(\st,\act)$.
\end{itemize}
Fairness assumptions are needed to give the learning algorithm sufficient opportunities to detect problems and try better policies~\cite{francez_1986}. 
Intuitively, in both choice points, the choice should be independent of what the policy and working memory say about states other than the current state.
This intuition is related to the Markov assumption, or \emph{independence of path} assumption~\cite{sutton-barto_1998}. Below, we formalize this intuition as a fairness notion for the operational semantics of Section~\ref{sub:opsem}.

We say that a run $\run$ is \emph{fair} if for each configuration $\cnf$ that occurs infinitely often at non-terminal positions, for each $(\act,\st')\in\opt\cnf$, the following transition occurs infinitely often:
\[
    \cnf
        \jump{\act\cs\st'}
    \apply\cnf\act{\st'}.
\]
We say that a task $\task$ is \emph{\learnableF} if all fair runs of $\task$ converge.

\begin{remark}
    \newcommand{\mycount}[1]{\fname{count}(#1)}
    There is always a fair run for any task, as follows.
    For each possible configuration $\cnf$, we could conceptually order the set $\opt{\cnf}$. 
    During a run, we could also keep track for each occurrence $i\in\natzero$ of a configuration $\cnf$ how many times we have already seen configuration $\cnf$ in the run, excluding the current occurrence; we denote this number as $\mycount i$.
    
    We begin the first trial with a random start configuration $\cnfX 1$, i.e., we choose a random start state and a random policy.
    We next choose the option $(\actX 1,\stX 2)\in\opt{\cnfX 1}$ with the first ordinal in the now ordered set $\opt{\cnfX 1}$.
    Now, for all the subsequent occurrences $i$ of a configuration $\cnf$ in the run, we choose the option with ordinal $j=(\mycount{i}\,\text{mod}\,\ssize{\opt{\cnf}})+1$ in the set $\opt{\cnf}$.
    So, if a configuration occurs infinitely often at non-terminal positions then we continually rotate through all its options.
    Naturally, trials end at the first occurrence of reward, and then we choose another start state; taking care to use all start states infinitely often.
    \qed
\end{remark}

\section{Results}
\label{sec:results}

The cycle-detection learning algorithm formalized in Section~\ref{sub:opsem} continually marks the encountered states as visited. 
At the end of trials, i.e., after obtaining reward, each state is again marked as unvisited. 
If the algorithm encounters a state $\st$ that is already visited within the same trial, the algorithm proposes to generate a new action for $\st$. 
Intuitively, if the same state $\st$ is encountered in the same trial, the agent might be running around in cycles and some new action should be tried for $\st$ to escape from the cycle.
It is important to avoid cycles if we want to achieve an eventual upper bound on the length of a trial, i.e., an upper bound on the time it takes to reach reward from a given start state.

Repeatedly trying a new action for revisited states might eventually lead to reward, and thereby terminate the trial.
In this learning process, the nondeterminism of the task can be both helpful and hindering: 
    nondeterminism is helpful if transitions choose successor states that are closer to reward, 
    but nondeterminism is hindering if transitions choose successor states that are further from reward or might lead to a cycle.
Still, on some suitable tasks, like reducible tasks, the actions that are randomly tried upon revisits might eventually globally form a policy that will never get trapped in a cycle ever again (see Theorem~\ref{theo:algvisit} below).

The outline of this section is as follows.
In Section~\ref{sub:sufficient}, we present a sufficient condition for tasks to be \learnableF. 
In Section~\ref{sub:final} we discuss how a simulator could detect that convergence has occurred in a fair run.
In Section~\ref{sub:necessary} we present necessary conditions for tasks to be \learnableF.

\subsection{Sufficient Condition for Convergence}
\label{sub:sufficient}

Intuitively, if a task is reducible then we might be able to obtain a policy that on each start state leads to reward without revisiting states in the same trial.
As long as revisits occur, we keep searching for the acyclic flow of states that is implied by reducibility.
We can imagine that states near the goal states, i.e., near immediate reward, tend to more quickly settle on an action that leads to reward. 
Subsequently, states that are farther removed from immediate reward can be reduced to states near goal states, and this growth process propagates through the entire state space.
This intuition is confirmed by the following convergence result:

\begin{theorem}
    \label{theo:algvisit}
    All reducible tasks are \learnableF.
\end{theorem}
\begin{proof}
    Let $\task=\tasktup$ be a reducible task.   
    Let $\run$ be a fair run on $\task$. 
    We show convergence of $\run$. 
    In Part~1 of the proof, we show that all trials in $\run$ terminate (with reward). 
    In Part~2, we show that eventually all states become stable in $\run$.
        
    \proofpara{Part 1: Trials terminate}
    Let 
    \[
        \rl 1,\, \rl 2,\, \ldots
    \]
    be the reducibility layers for $\task$ as defined in Section~\ref{sub:task}, where $\rl 1=\goals\task$.
    Let $\trial$ be a trial in $\run$.
    To show finiteness of $\trial$, and thus termination of $\trial$, we show by induction on $i=1, 2, \ldots$ that the states in $\rl i$ occur only finitely many times in $\trial$.
    Because $\task$ is reducible, there is an index $n\in\natzero$ for which $\rl{n}=\rl{n+1}=\states$, and therefore our inductive proof shows that every state only occurs a finite number of times in the trial $\trial$; hence, $\trial$ is finite.
    
    Before we continue, we recall that a state $\st$ is marked as visited after its first occurrence in the trial; any occurrence of $\st$ after its first occurrence is therefore in a branching configuration.
        
    \proofsubpara{Base case}
    Let $\st\in\rl 1=\goals\task$.
    Towards a contradiction, suppose $\st$ occurs infinitely often in trial $\trial$, making $\trial$ infinite.
    Because there are only a finite number of possible configurations, there is a configuration $\cnf$ containing $\st$ that occurs infinitely often in $\trial$ at non-terminal positions (because the trial is now infinite).
    Configuration $\cnf$ is branching because it occurs more than once.%
        \footnote{To see this, take for instance the second occurrence of $\cnf$ in the trial. That occurrence represents a revisit to $\st$, so $\st$ is in the working memory set of $\cnf$.}
    By definition of $\st\in\goals\task$, there is an action $\act\in\actions$ such that $(\st,\act)\in\rewards$. 
    Since always $\tr(\st,\act)\neq\emptyset$, we can choose some $\st'\in\tr(\st,\act)$. 
    We have $(\act,\st')\in\opt\cnf$ because $\cnf$ is branching.
    By fairness, the following transition must occur infinitely often in the trial:
    \[
        \cnf 
            \jump{\act\cs\st'}
        \apply\cnf\act{\st'}.
    \]
    But this transition is a reward transition, so the trial would have already ended at the first occurrence of this transition. Hence $\st$ can not occur infinitely many times; this is the desired contradiction.
        
    \proofsubpara{Inductive step}    
    Let $i\geq 2$, and let us assume that states in $\rl{i-1}$ occur only finitely many times in $\trial$.
    Let $\st\in\rl{i}\setminus\rl{i-1}$. 
    By definition of $\rl i$, there is some action $\act\in\actions$ such that $\tr(\st,\act)\subseteq\rl{i-1}$.
    Towards a contradiction, suppose $\st$ occurs infinitely often in $\trial$, making $\trial$ infinite.
    Like in the base case, there must be a branching configuration $\cnf$ containing $\st$ that occurs infinitely often in the trial (at non-terminal positions).
    Since always $\tr(\st,\act)\neq\emptyset$, we can choose some $\st'\in\tr(\st,\act)\subseteq\rl{i-1}$.
    We have $(\act,\st')\in\opt\cnf$ because $\cnf$ is branching.
    By fairness, the following transition must occur infinitely often in the trial:
    \[
        \cnf 
            \jump{\act\cs\st'}
        \apply\cnf\act{\st'}.
    \]
    But then $\st'$ would appear infinitely often in trial $\trial$. This is the desired contradiction, because the induction hypothesis says that all states in $\rl{i-1}$ (including $\st'$) occur only finitely many times in $\trial$.
        
    \proofpara{Part 2: Stability of states}    
    We now show that all states eventually become stable in the fair $\run$.       
    Let 
    \[
        \rl 1,\, \rl 2,\, \ldots
    \]
    again be the reducibility layers for $\task$ as above, where $\rl 1=\goals\task$.
    We show by induction on $i=1,2,\ldots$ that states in $\rl i$ become stable in $\run$.
    Since $\task$ is reducible, there is an index $n\in\natzero$ such that $\rl n=\rl{n+1}=\states$, so our inductive proof shows that all states eventually become stable.
    
    Before we continue, we recall that Part~1 of the proof has shown that all trials are finite. So, whenever we say that a configuration occurs infinitely often in the run, this means that the configuration occurs in infinitely many trials. 
    Similarly, if a transition occurs infinitely often in the run, this means that the transition occurs in infinitely many trials.
    
    \proofsubpara{Base case}    
    Let $\st\in\rl 1=\goals\task$.         
    Towards a contradiction, suppose $\st$ would not become stable.
    This means that there are infinitely many non-terminal occurrences of branching configurations containing $\st$.%
        \footnote{For completeness, we recall that if $\st$ would occur only a finite number of times in the run then we can immediately see in the definition of stability that $\st$ becomes stable.}
    Because there are only finitely many possible configurations, there must be a branching configuration $\cnf$ containing $\st$ that occurs infinitely often at non-terminal positions.
    
    By definition of $\st\in\goals\task$, there is an action $\act\in\actions$ such that $(\st,\act)\in\rewards$.
    Since always $\tr(\st,\act)\neq\emptyset$, we can choose some $\st'\in\tr(\st,\act)$.
    We have $(\act,\st')\in\opt\cnf$ because $\cnf$ is branching.
    By fairness, the following transition $\trans$ must occur infinitely often in the run:
    \[
        \cnf 
            \jump{\act\cs\st'}
        \apply\cnf\act{\st'}.
    \]
    Transition $\trans$ is a reward transition because $(\st,\act)\in\rewards$. 
    Let $j$ be the index of a trial containing transition $\trans$; this implies that $\trans$ is the last transition of trial $j$.
    We now show that any non-terminal occurrences of $\st$ after trial $j$ must be in a non-branching configuration. Hence, $\st$ becomes stable; this is the desired contradiction.
    
    Consider the first trial index $k$ after $j$ in which $\st$ occurs again at a non-terminal position.    
    Let configuration $\cnfX 1=\cnftupX 1$ with $\stX 1=\st$ be the first occurrence of $\st$ in trial $k$ (which is at a non-terminal position). 
    Note that $\polX 1(\st)=\act$ because 
        \romI\ trial $j$ ends with the assignment of action $\act$ to $\st$ (through transition $\trans$), and 
        \romII\ the trials between $j$ and $k$ could not have modified the action of $\st$.
    Further, configuration $\cnfX 1$ is not branching because $\st$ is not yet flagged as visited at its first occurrence in trial $k$.
    This means that at any occurrence of $\cnfX 1$, trial $k$ must select an option $(\act,\st'')\in\opt{\cnfX 1}$, with action $\act$ and $\st''\in\tr(\st,\act)$, and perform the corresponding transition $\trans'$:
    \[
        \cnfX 1
            \jump{\act\cs\st''}
        \apply{\cnfX 1}\act{\st''}.
    \]
    Again, since $(\st,\act)\in\rewards$, trial $k$ ends directly after transition $\trans'$; no branching configuration containing $\st$ can occur in trial $k$ at a non-terminal position.%
        \footnote{Although it is possible that $\st$ is directly revisited from itself, it does not matter whether the terminal configuration of the trial is branching or not.}
    This reasoning can now be repeated for all following trials to see that there are no more non-terminal occurrences of branching configurations containing $\st$.

    \proofsubpara{Inductive step}
    Let $i\geq 2$. 
    We assume for each $\st'\in\rl{i-1}$ that $\st'$ eventually becomes stable.
    Now, let $\st\in\rl{i}\setminus\rl{i-1}$.
    By definition of $\rl i$, there is an action $\act\in\actions$ such that $\tr(\st,\act)\subseteq\rl{i-1}$.
    Towards a contradiction, suppose that $\st$ does not become stable.
    Our aim is to show that now also at least one $\st'\in\tr(\st,\act)$ does not become stable, which would contradict the induction hypothesis.   
    
    Regarding terminology, we say that a chain is a \emph{$(\st,\act)$-chain} if \romI\ the chain contains only non-reward transitions and \romII\ the chain has the following desired form:
    \[
        \cnfX 1
            \jump{\actX 1\cs\stX 2}
        \cnfX 2
            \jump{\actX 2\cs\stX 3}
        \ldots
            \jump{\actX{n-1}\cs\stX n}
        \cnfX n,
    \]
    denoting $\cnfX j=\cnftupX j$ for each $j\in\set{1,\ldots,n}$, where 
        $\stX 1=\stX n=\st$ and
        $\actX 1=\act$.
    Note that such a chain starts and ends with an occurrence of $\st$, so $\st$ is revisited in the chain. Moreover, the first transition performs the action $\act$ from above.        
    Next, we say that a trial is a \emph{$(\st,\act)$-trial} if the trial contains a $(\st,\act)$-chain.
    In principle, each $(\st,\act)$-trial could embed a different $(\st,\act)$-chain.
    
    To see that there are infinitely many occurrences of $(\st,\act)$-trials in $\run$, we distinguish between the following two cases.
    
    \begin{itemize}
        \item Suppose that in $\run$ there are infinitely many occurrences of trials that end with a policy $\pol$ where $\pol(\st)=\act$, i.e., action $\act$ is assigned to $\st$. 
        Let $j$ be the index of such a trial occurrence. 
        Because by assumption $\st$ does not become stable, we can consider the first trial index $k$ after $j$ in which $\st$ occurs in a branching configuration at a non-terminal position. 
        Note that trials between trial $j$ and trial $k$ do not modify the action of $\st$.
        Now, the first occurrence of $\st$ in trial $k$ is always non-branching, and thus we perform action $\act$ there. 
        The subsequence in trial $k$ starting at the first occurrence of $\st$ and ending at some branching configuration of $\st$ at a non-terminal position, is a $(\st,\act)$-chain: the chain starts and ends with $\st$, its first transition performs action $\act$, and it contains only non-reward transitions because it ends at a non-terminal position.
        Hence, trial $k$ is a $(\st,\act)$-trial.
        
        \item Conversely, suppose that in $\run$ there are only finitely many occurrences of trials that end with a policy $\pol$ where $\pol(\st)=\act$.
        Let $\run'$ be an (infinite) suffix of $\run$ in which no trial ends with action $\act$ assigned to $\st$.
        Because by assumption $\st$ does not become stable, and because the number of possible configurations is finite, there is a branching configuration $\cnf$ containing $\st$ that occurs infinitely often at non-terminal positions in $\run'$.
        Choose some $\st'\in\tr(\st,\act)$.
        We have $(\act,\st')\in\opt{\cnf}$ because $\cnf$ is branching.
        By fairness, the following transition $\trans$ occurs infinitely often in $\run'$:
        \[
            \cnf 
                \jump{\act\cs\st'}
            \apply\cnf\act{\st'}.
        \]
        Let $j$ be the index of a trial occurrence in $\run'$ that contains transition $\trans$; there are infinitely many such indexes because all trials in $\run$ are finite (see Part 1 of the proof).
        Since transition $\trans$ attaches action $\act$ to $\st$, we know by definition of $\run'$ that any occurrence of $\trans$ in trial $j$ is followed by at least one other transition from state $\st$ that attaches an action $\actB$ to $\st$ with $\actB\neq\act$; this implies that after each occurrence of transition $\trans$ in trial $j$ there is a branching configuration of $\st$ at a non-terminal position.
        In trial $j$, a subsequence starting at any occurrence of $\trans$ and ending with the first subsequent branching configuration of $\st$ at a non-terminal position, is a $(\st,\act)$-chain: the chain starts and ends with $\st$, its first transition performs action $\act$, and the chain contains only non-reward transitions because it ends at a non-terminal position.
        Hence, trial $j$ is a $(\st,\act)$-trial.
    \end{itemize}
    
    We have seen above that there are infinitely many occurrences of $(\st,\act)$-trials in $\run$.
    Because there are only a finite number of possible configurations, there is a configuration $\cnf$ containing $\st$ that is used in infinitely many occurrences of $(\st,\act)$-trials as the last configuration of a $(\st,\act)$-chain. 
    Note that $\cnf$ occurs infinitely often at non-terminal positions since $(\st,\act)$-chains contain no reward transitions.
        
    Next, we can choose from some occurrence of a $(\st,\act)$-trial in the run some $(\st,\act)$-chain $\chain$ where in particular the last configuration of $\chain$ is the configuration $\cnf$.
    Formally, we write $\chain$ as
    \[
        \cnfX 1
            \jump{\actX 1\cs\stX 2}
        \cnfX 2
            \jump{\actX 2\cs\stX 3}
        \ldots
            \jump{\actX{n-1}\cs\stX n}
        \cnfX n,
    \]
    where $\cnfX n=\cnf$, and denoting $\cnfX j=\cnftupX j$ for each $j\in\set{1,\ldots,n}$, where 
        $\stX 1=\stX n=\st$ and
        $\actX 1=\act$. 
        We recall that all transitions of $\chain$ are non-reward transitions.
    Note that $2<n$: we have $\st\neq\stX 2$ because $\st\notin\rl{i-1}$ and $\stX 2\in\tr(\st,\act)\subseteq\rl{i-1}$.
    
    In chain $\chain$, we have certainly marked state $\st$ as visited after its first occurrence, causing configuration $\cnfX n$ to be branching.
    This implies $(\act,\stX 2)\in\opt{\cnfX n}$, where $(\act,\stX 2)$ is the same option as taken by the first transition of $\chain$, since $\actX 1=\act$.
    Also, since $2<n$, we have certainly marked state $\stX 2$ as visited after its first occurrence in $\chain$; this implies $\stX 2\in\wmX n$.
    Next, since the configuration $\cnfX n=\cnf$ occurs infinitely often at non-terminal positions (see above), the following transition also occurs infinitely often by fairness:
    \[
        \cnfX n
            \jump{\act\cs\stX 2}
        \cnfX{n+1},
    \]
    where $\cnfX{n+1}=\cnftupX{n+1}=\apply{\cnfX n}{\act}{\stX 2}$.
    Because $\stX{n+1}=\stX 2$ and $\stX 2\in\wmX n\subseteq\wmX{n+1}$, configuration $\cnfX{n+1}$ is branching.
    Moreover, we know that $(\st,\act)\notin\rewards$ since no transition of $\chain$ is a reward transition, including the first transition.
    So, the branching configuration $\cnfX{n+1}$ occurs infinitely often at non-terminal positions.
    Hence, $\stX 2$ would not become stable. 
    Yet, $\stX 2\in\tr(\st,\act)\subseteq\rl{i-1}$, and the induction hypothesis on $\rl{i-1}$ says that $\stX 2$ does become stable; this is the desired contradiction.
    \qedhere
\end{proof}

\begin{remark}
\label{remark:theo-algvisit}
By Theorem~\ref{theo:algvisit}, the trials in a fair run on a reducible task eventually contain a number of non-terminal configurations that is at most the number of states; otherwise at least one state would never become stable.%
    \footnote{If there would be infinitely many trials that contain more non-terminal configurations than states, then in infinitely many trials there is a revisit to a state (in a branching configuration) on a non-terminal position.
    Since there are only finitely many states, there would be at least one state $\st$ that in infinitely many trials occurs in a branching configuration on a non-terminal position; this state $\st$ does not become stable by definition.}
So, we get a relatively good eventual upper bound on trial length. 
However, Theorem~\ref{theo:algvisit} provides no information on the waiting time before that upper bound will emerge, because that waiting time strongly depends on the choices made by the run regarding start states of trials, tried actions, and successor states (see also Section~\ref{sec:conclusion}).

Because we seek a policy that avoids revisits to states in the same trial, an important intuition implied by Theorem~\ref{theo:algvisit} is that for reducible tasks eventually the trials of a run follow paths without cycles through the state space.
    The followed paths are still influenced by nondeterminism, but they never contain a cycle. 
    Also, a path followed in a trial is not necessarily the shortest possible path to reward, because the discovery of paths depends on experience, i.e., on the order in which actions were tried during the learning process. The experience dependence was experimentally observed, e.g.\ by \citet{fremaux_2013}.
\qed
\end{remark}

\begin{remark}
    The order in which states become stable in a fair run does not necessarily have to follow the order of the reducibility layers of Section~\ref{sub:task}. 
    In general, it seems possible that some states that are farther removed from goal states could become stable faster than some states nearer to goal states; but, to become stable, the farther removed states probably should first have some stable strategy to the goal states.
    
    To see that simulations do not exactly follow the inductive reasoning of the proof of Theorem~\ref{theo:algvisit}, one could compare, in the later Section~\ref{sec:examples}, the canonical policy implied by reducibility in Figure~\ref{fig:corridor_template} with an actual final policy in Figure~\ref{fig:forward_backward}.
    \qed
\end{remark}

The following example illustrates the necessity of the fairness assumption in Theorem~\ref{theo:algvisit}.
So, although the convergence result for reducible tasks appears natural, the example reveals that subtle notions, like the fairness assumption, should be taken into account to understand learning.
\begin{example}
    \newcommand{\cf}[3]{(#1,#2,\set{#3})}
    \label{ex:markov}    
    Consider again the task $\task$ from Example~\ref{ex:reduce}, that is also visualized in Figure~\ref{fig:reduce}.    
    In the following, for ease of notation, we will denote configurations as triples $(x,y,Z)$, where 
        $x$ is the current state; 
        $y$ is the action assigned by the policy to the specific state $1$, with action $a$ assigned to all other states; 
        and $Z$ is the set of visited states as before.
    
    Consider now the following trial $\trial_a$ where the initial policy has assigned action $a$ to all states, including the start state $1$:
    \[
        \cf 1 a ~
            \jump{a, 1}
        \cf 1 a {1}
            \jump{b, 2}
        \cf 2 b {1}
            \jump{a, 3}
        \cf 3 b {1,2}
            \jump{a, 3}
        \cf 3 b {1,2,3}.
    \]
    This is indeed a valid trial because the last transition is a reward transition. Note also that a revisit to state $1$ occurs in the first transition. The configuration $\cf 1 a {1}$ is thus branching, which implies that the option $(b, 2)$ may be chosen there.
    At the end of trial $\trial_a$, action $b$ is assigned to state $1$ and action $a$ is assigned to the other states.
    
    Consider also the following trial $\trial_b$ where the initial policy has assigned action $b$ to state $1$ and $a$ to all other states:
    \[
        \cf 1 b ~
            \jump{b, 2}
        \cf 2 b {1}
            \jump{a, 1}
        \cf 1 b {1,2}
            \jump{a, 3}
        \cf 3 a {1,2}
            \jump{a, 3}
        \cf 3 a {1,2,3}.
    \]
    The last transition is again a reward transition. Note that a revisit occurs to state $1$ in the second transition. The configuration $\cf 1 b {1,2}$ is therefore branching, which implies that the option $(a,3)$ may be chosen there.
    At the end of trial $\trial_b$, action $a$ is assigned to all states, including state $1$.
    
    Now, let $\run$ be the run that alternates between trials $\trial_a$ and $\trial_b$ and that starts with trial $\trial_a$.        
    The state $1$ never becomes stable in $\run$ because we assign action $a$ and action $b$ to state $1$ in an alternating fashion. 
    So, run $\run$ does not converge because there are infinitely many non-terminal occurrences of branching configurations containing state $1$.

    Although run $\run$ satisfies all requirements of a valid run, $\run$ is not fair.
    For example, although the configuration $\cf 1 b {1,2}$ occurs infinitely often (due to repeating trial $\trial_b$), this configuration is never extended with the valid option $(b, 2)$ that could propagate revisits of state $1$ to revisits of state $2$ in the same trial; revisits to state $2$ could force state $2$ to use the other action $b$, which in turn could aid state $1$ in becoming stable.        
    
    In conclusion, because task $\task$ is reducible and yet the valid (but unfair) run $\run$ does not converge, we see that Theorem~\ref{theo:algvisit} does not hold in absence of fairness.
    \qed
\end{example}

\subsection{Detecting the Final Policy}
\label{sub:final}

We refer to the lastly formed policy of a run as the final policy.
For an increased understanding of what convergence means, it appears interesting to say something about the form of the final policy. 
In particular, we would like to understand what kind of paths are generated by the final policy.
As an additional benefit, recognizing the form of the final policy allows us to detect the convergence point in a simulation.%
    \footnote{Precise convergence detection is possible because our framework does not model reward numerically and thus there are no numerical instability issues near convergence. The convergence detection enables some of the simulation experiments in Section~\ref{sec:examples}.}

We syntactically characterize the final policy in Theorem~\ref{theo:final}.
In general, verifying the syntactical property of the final policy requires access to the entire set of task states.
In this subsection, we do not require that tasks are reducible.

We first introduce the two key parts of the syntactical characterization, namely, the so-called \emph{forward} and \emph{backward} sets of states induced by a policy.
As we will see below, the syntactical property says that the forward set should be contained in the backward set.
%
\paragraph{Forward and Backward}
Let $\task=\tasktup$ be a task that is \learnableF. 
To make the notations below easier to read, we omit the symbol $\task$ from them. It will always be clear from the context which task is meant.

Let $\pol:\states\to\actions$ be a policy, i.e., each $\st\in\states$ is assigned an action from $\actions$.
First, we define 
\[
    \ground=\set{\st\in\goals\task\mid(\st,\pol(\st))\in\rewards};
\] this is the set of all goal states that are assigned a rewarding action by the policy.
Next, we define two sets $\forward\subseteq\states$ and $\backward\subseteq\states$, as follows.
For the set $\forward$, we consider the infinite sequence $\fl 1$, $\fl 2$, \ldots\ of sets, where $\fl 1=\initstates$ and for each $i\geq 2$,
\[
    \fl i = \fl{i-1} \cup 
                \bigcup_{\mathlarger{\st\in\fl{i-1}\setminus\ground}}\tr(\st,\pol(\st)).
\]
We define $\forward=\bigcup_{i\in\natzero}\fl i$.
Note that $\forward\subseteq\states$.
Intuitively, the set $\forward$ contains all states that are reachable from the start states by following the policy. 
In the definition of $\fl{i}$ with $i\geq 2$, we remove $\ground$ from the extending states because we only want to add states to $\forward$ that can occur at non-terminal positions of trials.%
    \footnote{Possibly, some states directly reachable from $\ground$ are still in $\forward$ because those states are also reachable from states outside $\ground$.}

For the set $\backward$, we consider the infinite sequence $\bl 1$, $\bl 2$, \ldots\ of sets, where $\bl 1=\ground$ and for each $i\geq 2$,
\[
    \bl i = \bl{i-1} \cup \set{\st\in\states\mid\tr(\st,\pol(\st))\subseteq\bl{i-1}}.
\]
We define $\backward=\bigcup_{i\in\natzero}\bl i$. Note that $\backward\subseteq\states$.
Intuitively, $\backward$ is the set of all states that are reduced to the goal states in $\ground$ by the policy.

For completeness, we remark that the infinite sequences $\fl 1$, $\fl 2$, \ldots, and $\bl 1$, $\bl 2$, \ldots, each have a fixpoint because $\states$ is finite.

\paragraph{Final Policy}

We formalize the final policy.
Let $\task$ be a task that is \learnableF.
Let $\run$ be a fair run on $\task$, which implies that $\run$ converges.
We define the \emph{convergence-trial} of $\run$ as the smallest trial index $i$ for which the following holds: trial $i$ terminates and after trial $i$ there are no more branching configurations at non-terminal positions.%
    \footnote{With this definition of convergence-trial, a run converges if and only if the run contains a convergence-trial.}
This implies that after trial $i$ the policy can not change anymore, because to change the action assigned to a state $\st$, the state $\st$ would have to occur again in branching configuration at a non-terminal position.
We define the \emph{final policy} of $\run$ to be the policy at the end of the convergence-trial.
In principle, different converging runs can have different final policies.

Now, we can recognize the final policy with the following property, that intuitively says that any states reachable by the policy are also safely reduced by the policy to reward:

\begin{theorem}
    \label{theo:final}
    Let $\task$ be a task that is \learnableF.%
        \footnote{In contrast to Theorem~\ref{theo:algvisit}, we do not require that $\task$ is reducible.}
    Let $\run$ be a converging fair run of $\task$.
    A policy $\pol$ occurring in run $\run$ at the end of a trial is the final policy of $\run$ if and only if
    \[
    \forward\subseteq\backward.
    \]
\end{theorem}
\begin{proof}
We show in two separate parts that $\forward\subseteq\backward$ is \romI\ a sufficient and \romII\ a necessary condition for $\pol$ to be the final policy of run $\run$.

\proofpara{Part 1: Sufficient condition}
Let $\pol$ be a policy occurring in run $\run$ at the end of a trial.
Assume that $\forward\subseteq\backward$.
We show that $\pol$ is the final policy of $\run$.

Concretely, we show that any trial starting with policy $\pol$ will \romI\ use $\pol$ in all its configurations, including the terminal configuration; and, \romII, does not contain branching configurations at non-terminal positions. 
This implies that the first trial ending with $\pol$ is the convergence-trial, so $\pol$ is the final policy.

Let $\trial$ be a trial in $\run$ that begins with policy $\pol$.
We explicitly denote trial $\trial$ as the following finite chain of transitions:
\[
    \cnfX 1
        \jump{\actX 1\cs\stX 2}
    \ldots
        \jump{\actX{n-1}\cs\stX n}
    \cnfX n.
\]
For each $i\in\set{1,\ldots,n}$, we denote $\cnfX i=\cnftupX i$.
Let $\fl 1$, $\fl 2$, \ldots\ be the infinite sequence of sets previously defined for $\forward$.
We show by induction on $i=1,\ldots,n-1$ that
\begin{enumerate}[(a)]
    \item \label{enu:save-pol} $\polX i=\pol$;
    \item \label{enu:state-in-forward} $\stX i\in\fl i$;
    \item \label{enu:non-branching} $\cnfX i$ is non-branching.
\end{enumerate}
At the end of the induction, we can also see that $\polX n=\pol$: 
    first, we have $\polX n=\polX{n-1}$ because configuration $\cnfX{n-1}$ is non-branching by property~\refitem{enu:non-branching};%
        \footnote{Because configuration $\cnfX{n-1}$ is non-branching, we have $\polX{n}(\stX{n-1})=\actX{n-1}=\polX{n-1}(\stX{n-1})$, which, combined with $\polX{n}(\stB)=\polX{n-1}(\stB)$ for each $\stB\in\states\setminus\set{\stX{n-1}}$, gives $\polX{n}=\polX{n-1}$.}
    second, $\polX{n-1}=\pol$ by property~\refitem{enu:save-pol}.

\proofsubpara{Base case}
Let $i=1$.
For property~\refitem{enu:save-pol}, we have $\polX 1=\pol$ because the trial starts with policy $\pol$.
For property~\refitem{enu:state-in-forward}, we see that $\stX 1\in\initstates=\fl 1$.
For property~\refitem{enu:non-branching}, we know that $\cnfX 1$ is non-branching because the first configuration in a trial still has an empty working memory of visited states.

\proofsubpara{Inductive step}
Let $i\geq 2$, with $i\leq n-1$. Assume that the induction properties are satisfied for the configurations $\cnfX1$, \ldots, $\cnfX{i-1}$.
We now show that the properties are also satisfied for $\cnfX i$.

\begin{description}
    \item[Property~\refitem{enu:save-pol}]
    By applying the induction hypothesis for property~\refitem{enu:non-branching} to $\cnfX{i-1}$, namely that $\cnfX{i-1}$ is non-branching, we know $\polX i=\polX{i-1}$.
    By subsequently applying the induction hypothesis for property~\refitem{enu:save-pol} to $\cnfX{i-1}$, namely $\polX{i-1}=\pol$, we know $\polX i=\pol$.

    \item[Property~\refitem{enu:state-in-forward}]
    To start, we note that $\stX i\in\tr(\stX{i-1},\polX{i-1}(\stX{i-1}))$ because $\cnfX{i-1}$ is non-branching by the induction hypothesis for property~\refitem{enu:non-branching}.
    By subsequently applying the induction hypothesis for property~\refitem{enu:save-pol} to $\cnfX{i-1}$, namely $\polX{i-1}=\pol$, we know $\stX i\in\tr(\stX{i-1},\pol(\stX{i-1}))$.
    Moreover, since $i-1<i\leq n-1$, the transition $\cnfX{i-1}\jump{\actX{i-1}\cs\stX i}\cnfX i$, where $\actX{i-1}=\pol(\stX{i-1})$, is a non-reward transition. 
    Hence, $(\stX{i-1},\pol(\stX{i-1}))\notin\rewards$ and thus $\stX{i-1}\notin\ground$.    
    
    Lastly, by applying the induction hypothesis for property~\refitem{enu:state-in-forward} to $\cnfX{i-1}$, we overall obtain that $\stX{i-1}\in\fl{i-1}\setminus\ground$.
    Combined with $\stX i\in\tr(\stX{i-1},\pol(\stX{i-1}))$, we see that $\stX i\in\fl i$.

    \item[Property~\refitem{enu:non-branching}]
    Towards a contradiction, suppose that configuration $\cnfX i$ is branching. This means that state $\stX i$ is revisited in $\cnfX i$.%
        \footnote{Recall that, by definition of branching configuration, we have $\stX i\in\wmX i$.}
    Let $V=\set{\stX 1,\ldots,\stX{i-1}}$. Note that $\stX i\in V$, which implies $V\neq\emptyset$.
    By applying the induction hypothesis for property~\refitem{enu:state-in-forward} to configurations $\cnfX 1$, \ldots, $\cnfX{i-1}$, we know that $V\subseteq\forward$.
    We now show that $V\cap\backward=\emptyset$, which would imply $\forward\not\subseteq\backward$; this is the desired contradiction.
    
    Let $\bl 1$, $\bl 2$, \ldots be the infinite sequence of sets defined for $\backward$ above. We show by induction on $j=1,2,\ldots$ that $V\cap\bl j=\emptyset$, which then overall implies $V\cap\backward=\emptyset$.
    \begin{itemize}
        \item Base case: $j=1$. By definition, $\bl 1=\ground$. 
        Let $\st\in V$.
        Let $k\in\set{1,\ldots,i-1}$ be the smallest index for which $\stX k=\st$, i.e., configuration $\cnfX k$ represents the first occurrence of $\st$ in the trial.
        By applying the outer induction hypothesis for properties \refitem{enu:save-pol} and \refitem{enu:non-branching} to $\cnfX k$, we know that $\actX k=\polX k(\stX k)=\pol(\st)$.
        But since $k\leq i-1<i\leq n-1$, we know that transition $\cnfX k\jump{\actX k\cs\stX{k+1}}\cnfX{k+1}$ is not a reward transition, implying $(\st,\pol(\st))\notin\rewards$. Hence, $\st\notin\ground$, and overall $V\cap\ground=\emptyset$.
        
        \item Inductive step. Let $j\geq 2$. 
        Assume $V\cap\bl{j-1}=\emptyset$.
        Towards a contradiction, suppose $V\cap\bl j\neq\emptyset$. Take some $\st\in V\cap\bl j$.
        If $\st\in\bl{j-1}$ then we would immediately have a contradiction with the induction hypothesis.
        Henceforth, suppose $\st\in\bl{j}\setminus\bl{j-1}$, which, by definition of $\bl{j}$, means $\tr(\st,\pol(\st))\subseteq\bl{j-1}$.
        We will now show that $V\cap\tr(\st,\pol(\st))\neq\emptyset$, which would give $V\cap\bl{j-1}\neq\emptyset$; this is the desired contradiction.
        
        Since $\st\in V$, there is some smallest $k\in\set{1,\ldots,i-1}$ such that $\stX k=\st$.
        Using a similar reasoning as in the base case ($j=1$), by applying the outer induction hypothesis for properties \refitem{enu:save-pol} and \refitem{enu:non-branching} to configuration $\cnfX k$, we can see that $\actX k=\pol(\stX k)$.
        This implies $\stX{k+1}\in\tr(\stX{k},\pol(\stX k))$.
        As a last step, we show that $\stX{k+1}\in V$, which gives $V\cap\tr(\stX{k},\pol(\stX k))\neq\emptyset$.
        We distinguish between the following cases:
        \begin{itemize}
            \item If $k\leq i-2$ then $k+1\leq i-1$, and surely $\stX{k+1}\in V$ by definition of $V$.            
            \item If $k=i-1$ then we know $\stX{k+1}=\stX i\in V$ because configuration $\cnfX i$ revisits state $\stX i$ (see above).
        \end{itemize}
    \end{itemize}
    
\end{description}

\proofpara{Part 2: Necessary condition}
Let $\pol$ be the final policy of $\run$. 
We show that $\forward\subseteq\backward$.
By definition of final policy, $\pol$ is the policy at the end of the convergence-trial, whose trial index we denote as $i$.
By definition of convergence-trial, after trial $i$ there are no more branching configurations at non-terminal positions.
Note in particular that the policy no longer changes after trial $i$.

Towards a contradiction, suppose that $\forward\not\subseteq\backward$.
Let $V=\forward\setminus\backward$. Note that $V\neq\emptyset$.
We show that there is a state $\st\in V$ that occurs at least once in a branching configuration at a non-terminal position after the convergence-trial $i$; this would be the desired contradiction.

We provide an outline of the rest of the proof. The reasoning proceeds in two steps. First, we show for each $\st\in V$ that $\tr(\st,\pol(\st))\cap V\neq\emptyset$. This means that if we are inside set $V$, we have the option to stay longer inside $V$ if we follow the policy $\pol$. Now, the second step of the reasoning is to show that we can stay arbitrarily long inside $V$ even after the convergence-trial $i$, causing at least one state of $V$ to occur in a branching configuration at a non-terminal position after trial $i$.

\proofsubpara{Step 1}
Let $\st\in V$. We show $\tr(\st,\pol(\st))\cap V\neq\emptyset$.
Towards a contradiction, suppose that $\tr(\st,\pol(\st))\cap V=\emptyset$.
Our strategy is to show that $\tr(\st,\pol(\st))\subseteq\backward$, which, by definition of $\backward$, implies that there is some index $j\in\natzero$ such that $\tr(\st,\pol(\st))\subseteq\bl j$. Therefore $\st\in\bl{j+1}\subseteq\backward$. But that is false because $\st\in V=\forward\setminus\backward$; this is the desired contradiction.

We are left to show that $\tr(\st,\pol(\st))\subseteq\backward$.
First, we show $\tr(\st,\pol(\st))\subseteq\forward$.
By definition, $\forward=\bigcup_{j\in\natzero}\fl j$.
Since $\st\in V\subseteq\forward$, there is some index $j\in\natzero$ such that $\st\in\fl j$.
Moreover, since $\st\notin\backward$ and $\ground\subseteq\backward$, we have $\st\notin\ground$.
Overall, $\st\in\fl j\setminus\ground$, which implies that $\tr(\st,\pol(\st))\subseteq\fl{j+1}\subseteq\forward$.

Now, we can complete the reasoning by combining $\tr(\st,\pol(\st))\subseteq\forward$ with our assumption $\tr(\st,\pol(\st))\cap V=\emptyset$, to see the following:
\begin{align*}
    \tr(\st,\pol(\st)) 
                       & \subseteq \forward\setminus V\\
                       & = \forward\setminus(\forward\setminus\backward)\\
                       & = \forward\cap\backward\\
                       & \subseteq\backward.
\end{align*}

\proofsubpara{Step 2}
We now show that after convergence-trial $i$ there is at least one occurrence of a branching configuration at a non-terminal position.

We first show that each $\st\in\forward$ occurs infinitely often at non-terminal positions after trial $i$. Recall that $\forward=\bigcup_{j\in\natzero}\fl j$. 
We show by induction on $j=1,2,\ldots$ that states in $\fl j$ occur infinitely often at non-terminal positions after trial $i$.
\begin{itemize}
    \item Base case: $j=1$. By definition, $\fl 1=\initstates$.
    Because $\run$ is a valid run, each state of $\initstates$ is used in infinitely many trials as the start state, also after trial $i$ (see Section~\ref{sub:opsem}).%
        \footnote{We also recall here that $\run$ is assumed to converge, which, by definition of convergence, implies that $\run$ is an infinite sequence of finite trials.}
    Moreover, the first configuration of a trial is always at a non-terminal position because each trial contains at least one transition.
    
    \item Inductive step. Let $j\geq 2$. Assume that each state in $\fl{j-1}$ occurs infinitely often at non-terminal positions after trial $i$.
    Let $\st\in\fl j\setminus\fl{j-1}$.
    This implies that there is some $\stX{j-1}\in\fl{j-1}\setminus\ground$ for which $\st\in\tr(\stX{j-1},\pol(\stX{j-1}))$.
    
    By applying the induction hypothesis, we know that $\stX{j-1}$ occurs infinitely often at non-terminal positions after trial $i$. 
    Because there are only a finite number of possible configurations, there is a configuration $\cnf$ containing $\stX{j-1}$ that occurs infinitely often at non-terminal positions after trial $i$.
    Because trial $i$ is the convergence-trial, we can make two observations about configuration $\cnf$:         
        first, $\cnf$ contains the final policy $\pol$ because the policy no longer changes after trial $i$; and,
        second, $\cnf$ is non-branching because no branching configurations occur at non-terminal positions after trial $i$.
    These observations imply $(\pol(\stX{j-1}),\st)\in\opt\cnf$.
    
    Now, since $\cnf$ occurs infinitely often at non-terminal positions after trial $i$, the following transition $\trans$ occurs infinitely often after trial $i$ by fairness:
    \[
        \cnf
            \jump{\pol(\stX{j-1})\cs\st}
        \apply\cnf{\pol(\stX{j-1})}\st.
    \]
    Lastly, we know that $(\stX{j-1},\pol(\stX{j-1}))\notin\rewards$ because $\stX{j-1}\notin\ground$, so transition $\trans$ is a non-reward transition. 
    Therefore state $\st$ occurs infinitely often at non-terminal positions after trial $i$.
\end{itemize}

Now take some $\stX 1\in V$. Since $V\subseteq\forward$, we know from above that $\stX 1$ occurs infinitely often at non-terminal positions after trial $i$. Because there are only finitely many possible configurations, there is a configuration $\cnfX 1$ containing $\stX 1$ that occurs infinitely often at non-terminal positions after trial $i$.
After trial $i$, the policy no longer changes and only non-branching configurations may occur at non-terminal positions. So $\cnfX 1$ contains the final policy $\pol$ and is non-branching. 
Moreover, since $\stX 1\in V$, we know from Step~1 above that there is some $\stX 2\in\tr(\stX 1,\pol(\stX 1))\cap V$. 
Overall, we can see that $(\pol(\stX 1),\stX 2)\in\opt{\cnfX 1}$.
By fairness, the following transition $\trans_1$ occurs infinitely often after trial $i$:
\[
    \cnfX 1
        \jump{\pol(\stX 1)\cs\stX 2}
    \cnfX 2,
\]
where $\cnfX 2=\apply{\cnfX 1}{\pol(\stX 1)}{\stX 2}$. Because $\stX 1\in V$ we have $\stX 1\notin\ground$, so transition $\trans_1$ is a non-reward transition.%
    \footnote{If $\stX 1\in\ground$ then $\stX 1\in\backward$, which is false since $\stX 1\in V$.}
Therefore configuration $\cnfX 2$ occurs infinitely often at non-terminal positions after trial $i$.
Denoting $\cnfX 2=\cnftupX 2$, note that $\stX 1\in\wmX 2$.

We now make a similar reasoning for $\cnfX 2$ as we did for $\cnfX 1$.
Since $\stX 2\in V$, we know again from Step~1 that there is some $\stX 3\in\tr(\stX 2,\pol(\stX 2))\cap V$. Configuration $\cnfX 2$ contains the final policy $\pol$ because $\cnfX 2$ occurs after trial $i$, and $\cnfX 2$ is also non-branching because it occurs after trial $i$ at a non-terminal position.
Therefore $(\pol(\stX 2),\stX 3)\in\opt{\cnfX 2}$.
Now, since $\cnfX 2$ occurs infinitely often at non-terminal positions after trial $i$ (see above), the following transition occurs infinitely often after trial $i$ by fairness:
\[    
    \cnfX 2
        \jump{\pol(\stX 2)\cs\stX 3}
    \cnfX 3,
\]
where $\cnfX 3=\apply{\cnfX 2}{\pol(\stX 2)}{\stX 3}$. This transition is also non-rewarding because $\stX 2\in V$ implies $\stX 2\notin\ground$.
Denoting $\cnfX 3=\cnftupX 3$, note that $\set{\stX 1, \stX 2}\subseteq\wmX 3$. We emphasize that more states of $V$ are now marked as visited.

We can now complete the reasoning.
By following the final policy $\pol$ from a state in $V$, we can always stay inside $V$ without reaching reward. So, the above procedure can be repeated $\ssize V$ times in total, to show the existence of a configuration $\cnf=\cnftup$ with $\st\in V$ that occurs infinitely often at non-terminal positions after trial $i$, and where $\st\in\wm$.
Configuration $\cnf$ is therefore branching,
and thus the existence of $\cnf$ gives the desired contradiction, as explained at the beginning of Part~2 of the proof.%
    \footnote{We note the following for completeness. By repeatedly trying to stay inside $V$, we are not guaranteed to see all of $V$, but we still know that after at most $\ssize V$ transitions we have to revisit a state of $V$ at a non-terminal position (after trial $i$); that first revisit forms the desired contradiction.}
    \qedhere
\end{proof}

\subsection{Necessary Conditions for Convergence}
\label{sub:necessary}

In Section~\ref{sub:sufficient}, we have seen that reducibility is a sufficient property for tasks to be \learnableF\ (Theorem~\ref{theo:algvisit}). 
In this subsection, we also show necessary properties for tasks to be \learnableF. This provides a first step towards characterizing the tasks that are \learnableF.
Thinking about such a characterization is useful because it allows us to better understand the tasks that are not \learnableF.

We first introduce some auxiliary concepts. Let $\task=\tasktup$ be a task. 
We say that there is a \emph{(simple) path} from a state $\st$ to a state $\st'$ if there is a sequence of actions $\actX 1,\ldots,\actX n$, with possibly $n=0$, and a sequence of states $\stX 1,\ldots,\stX{n+1}$ such that 
\begin{itemize}
    \item $\stX 1=\st$; 
    \item $\stX{n+1}=\st'$; 
    \item $\stX{i+1}\in\tr(\stX{i},\actX{i})$ for each $i\in\set{1,\ldots,n}$; 
    \item $(\stX{i},\actX{i})\notin\rewards$ for each $i\in\set{1,\ldots,n}$; and,
    \item $\stX{i}\neq\stX{j}$ for each $i,j\in\set{1,\ldots,n+1}$ with $i\neq j$.
\end{itemize}
We emphasize that the path does not contain reward transitions and does not repeat states.
We also denote a path as
\[
    \stX 1
        \jump{\actX 1}
    \ldots
        \jump{\actX n}
    \stX{n+1}.
\]
Note that there is always a path from a state to itself, namely, the empty path where $n=0$.
    
We say that a state $\st$ is \emph{reachable} if there is a path from a start state $\stX 0\in\initstates$ to $\st$.
Next, letting $\stateset\subseteq\states$, we say that a state $\st$ \emph{has a path to $\stateset$} if there is a path from $\st$ to a state $\st'\in\stateset$.
Note that if $\st$ has a path to $\stateset$, it is not guaranteed that the action sequence of that path always ends in $\stateset$ because the transition function $\tr$ is nondeterministic. 

We can now note the following necessary properties for tasks to be \learnableF:
\begin{proposition}
    \label{prop:necessary}
    Tasks $\task$ that are \learnableF\ satisfy the following properties:
    \begin{enumerate}[(a)]
        \item \label{enu:path} The reachable states have a path to $\goals\task$.
        \item \label{enu:start-reduce} The start states are reducible to $\goals\task$.        
    \end{enumerate}
\end{proposition}
The property~\refitem{enu:path} is related to the assumption that reward can be reached from every state, see e.g.~\cite{bonet_2006}.
\begin{proof}
We show the two properties separately. Denote $\task=\tasktup$.

\proofpara{Property~\refitem{enu:path}}
Towards a contradiction, suppose that there is a reachable state $\st$ that has no path to $\goals\task$.
Note that $\st\notin\goals\task$ because otherwise the empty path from $\st$ to itself would be a path to $\goals\task$.
Because $\st$ is reachable, there is a path
\[
    \stX 1
        \jump{\actX 1}
    \ldots
        \jump{\actX n}
    \stX{n+1},
\]
where $\stX 1\in\initstates$ and $\stX{n+1}=\st$.
We can consider a policy $\pol$ where, for each $i\in\set{1,\ldots,n}$, we set $\pol(\stX{i})=\actX{i}$; the other state-action mappings may be arbitrary.%
    \footnote{Note that there are no conflicting assignments of actions to states because, by definition of path, each state $\stX i$ with $i\in\set{1,\ldots,n+1}$ occurs only once on the path.}
Below we will consider a fair run $\run$ whose first trial is given $\stX 1$ as start state and $\pol$ as the initial policy. 
First, we consider the following chain $\chain$:
\[
    \cnfX 1
        \jump{\actX 1\cs\stX 2}
    \ldots
        \jump{\actX n\cs\stX{n+1}}
    \cnfX{n+1},
\]
where for each $i\in\set{1,\ldots,n+1}$ we define $\cnfX i=(\stX i,\pol,\wmX i)$ where
    $\wmX i=\set{\stX 1,\ldots,\stX{i-1}}$.%
    \footnote{In this notation, we interpret $\set{\stX 1,\ldots,\stX 0}$ as $\emptyset$.}
Note that this chain is indeed valid: for each $i\in\set{1,\ldots,n}$, configuration $\cnfX i$ is (constructed to be) non-branching and therefore $(\pol(\stX i),\stX{i+1})=(\actX i,\stX{i+1})\in\opt{\cnfX{i}}$; this means that we do not modify the policy during the transition $\cnfX i\jump{\actX i\cs\stX{i+1}}\apply{\cnfX i}{\actX i}{\stX{i+1}}$, but we only mark $\stX i$ as visited, which gives $\cnfX{i+1}=\apply{\cnfX i}{\actX i}{\stX{i+1}}$.
Note that configuration $\cnfX{n+1}$ contains the state $\st$.

Consider a fair run $\run$ whose first trial starts with the chain $\chain$.
We show that the first trial never terminates; this is the desired contradiction because we had assumed that task $\task$ is \learnableF.
For the first trial to terminate, the trial must extend chain $\chain$ to a chain $\chain'$ that terminates with a reward transition. But then the existence of $\chain'$ would imply that there is a path from $\st$ to a state $\st'\in\goals\task$, which is false.

\proofpara{Property~\refitem{enu:start-reduce}}
We show that the start states are reducible to $\goals\task$.
Let $\run$ be a fair run on $\task$, which implies that $\run$ converges.
Because $\run$ converges, there is a final policy $\pol$, as defined in Section~\ref{sub:final}. 
Then by Theorem~\ref{theo:final}, we know that $\forward\subseteq\backward$.
We have $\initstates\subseteq\backward$ because $\initstates\subseteq\forward$ by definition of $\forward$.
Letting $\reduce$ be the set of states that are reducible to $\goals\task$, we show below that $\backward\subseteq\reduce$; this implies $\initstates\subseteq\reduce$, as desired.    

In Section~\ref{sub:task}, we have defined $\reduce=\bigcup_{i\in\natzero}\rl i$ where $\rl 1=\goals\task$ and for each $i\geq 2$,
\[
    \rl i = \rl{i-1} \cup \set{\st\in\states\mid\exists\act\in\actions\text{ with }\tr(\st,\act)\subseteq\rl{i-1}}.
\]
Also recall the definition $\backward=\bigcup_{i\in\natzero}\bl i$ from Section~\ref{sub:final}.
We show by induction on $i=1,2,\ldots$ that $\bl i\subseteq\rl i$.
\begin{itemize}
    \item Base case: $i=1$. By definition, $\bl 1=\ground=\set{\st\in\goals\task\mid(\st,\pol(\st))\in\rewards}$. Hence, $\bl 1\subseteq\goals\task=\rl 1$.
    
    \item Inductive step. 
    Let $i\geq 2$, and let us assume that $\bl{i-1}\subseteq\rl{i-1}$.
    Let $\st\in\bl i\setminus\bl{i-1}$, which, by definition of $\bl i$, implies $\tr(\st,\pol(\st))\subseteq\bl{i-1}$.
    By applying the induction hypothesis, we see $\tr(\st,\pol(\st))\subseteq\rl{i-1}$.
    This implies $\st\in\rl{i}$, and, overall, $\bl i\subseteq\rl i$.    
\end{itemize}

This completes the proof.
\qedhere
\end{proof}

We recall from Theorem~\ref{theo:algvisit} that reducibility is a sufficient property for tasks to be \learnableF.
Note that reducibility implies the necessary properties in Proposition~\ref{prop:necessary} for tasks to be \learnableF.
We now discuss the gap between these sufficient and necessary properties.
Let $\task=\tasktup$ be a task.
Letting $\reduce$ be the set of states that can be reduced to $\goals\task$ (as in Section~\ref{sub:task}), we define 
\[
    \unstable=\states\setminus\reduce.
\]
We call $\unstable$ the \emph{unstable set} of $\task$. Note that for each state $\st\in\unstable$, for each action $\act\in\actions$, we always have $\tr(\st,\act)\cap\unstable\neq\emptyset$.%
    \footnote{Indeed, if $\tr(\st,\act)\cap\unstable=\emptyset$ then $\tr(\st,\act)\subseteq\reduce$ and subsequently $\st\in\reduce$, which is false.}
So, once we are inside $\unstable$, we can never reliably escape $\unstable$: escaping $\unstable$ depends on the nondeterministic choices regarding successor states.
This intuition has also appeared in Part~2 of the proof of Theorem~\ref{theo:final}, but in that proof we were focusing on just a single fixed action for each state, as assigned by the final policy at hand.

The following example illustrates how a nonempty unstable set could prevent convergence. 
In particular, the example illustrates that the necessary properties of Proposition~\ref{prop:necessary} are not sufficient for a task to be \learnableF.
\begin{example}
    \label{ex:unstable}    
    Consider the task $\task=\tasktup$ defined as follows:    
        $\states=\set{1, 2, 3}$;
        $\initstates=\set{1}$;
        $\actions=\set{a, b}$;
        $\rewards=\set{(3,a), (3,b)}$; and, regarding $\tr$, we define
        \begin{align*}
            \tr(1,a) &=\set{2},\\
            \tr(1,b) &=\set{3},\\
            \tr(2,a)=\tr(2,b)&=\set{2, 3},\\
            \tr(3,a)=\tr(3,b)&=\set{3}.
        \end{align*}
    The task $\task$ is visualized in Figure~\ref{fig:unstable}.
    Note that $\reduce=\set{1,3}$, giving 
    \[
    \unstable=\set{2}.
    \]
    Note in particular that $\tr(2,a)\cap\unstable\neq\emptyset$ and $\tr(2,b)\cap\unstable\neq\emptyset$.
    
    Task $\task$ satisfies the necessary properties of Proposition~\ref{prop:necessary}: 
        \romI\ the reachable states, which are all states in this case, have a path to $\goals\task$, and
        \romII\ the start state $1$ is reducible to $\goals\task$.         
    However, task $\task$ is not \learnableF, as we now illustrate. 
    Consider a trial $\trial$ of the following form: starting with an initial policy that assigns action $a$ to state $1$, we first go from state $1$ to state $2$; next, we stay at least two consecutive times in state $2$; and, lastly, we proceed to state $3$ and obtain reward there. 
    Because there are no revisits to state $1$ in trial $\trial$, state $1$ remains connected to action $a$. 
    We now see that we can make a fair run $\run$ by repeating trials of the form of $\trial$: state $1$ is never revisited and stays connected to action $a$, and we keep revisiting state $2$.    
    This way, there are infinitely many branching configurations containing state $2$ at non-terminal positions.
    So, run $\run$ does not converge.
    \qed
\end{example}

\begin{figure}
    \begin{center}
    \includegraphics[width=0.5\textwidth]{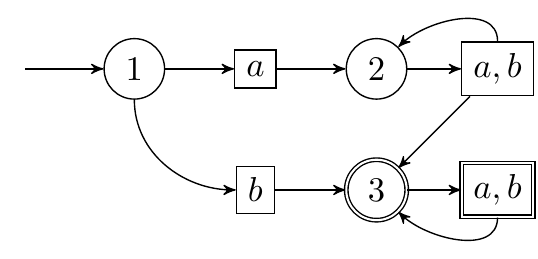}
    \end{center}
    \caption{The task from Example~\ref{ex:unstable}, with unstable set $\set{2}$. The graphical notation is explained in Figure~\ref{fig:reduce}.}
    \label{fig:unstable}
\end{figure}

By looking at Example~\ref{ex:unstable}, it seems that runs would somehow have to learn to avoid entering the set $\unstable$.
First, we define
    \[    
    \border=\set{\st\in\reduce\mid \exists\act\in\actions\text{ with }\tr(\st,\act)\cap\unstable\neq\emptyset}.
    \]
The set $\border$ contains those states that could enter $\unstable$; we call such states \emph{border states}.
Let $\st\in\border$, and let $V$ be the subset of $\unstable$ that is reachable from $\st$. Now, one idea could be to demand for each $\stX 1\in V$ that there is some $(\stX 2,\act)\in V\times\actions$ such that $\stX 2\in V$ is reachable from $\stX 1$ and $\st\in\tr(\stX 2,\act)$, i.e., there is some escape option to return from $V$ back to border state $\st$. 
This way, we can revisit that precise border state $\st$ in the same trial, so that under fairness we can choose a new action for $\st$ to avoid a future entrance into $\unstable$.
The possibility to revisit border states in the same trial is exactly what is missing from Example~\ref{ex:unstable}.
The characterization of tasks that are \learnableF\ might be a class of tasks that satisfy the necessary properties of Proposition~\ref{prop:necessary} and that additionally specify assumptions on the transition function to ensure the possibility of revisits to border states.
Illuminating the role of unstable sets in learning is an interesting avenue for further work (see Section~\ref{sec:conclusion}).

\section{Examples}
\label{sec:examples}

Theorem~\ref{theo:algvisit} tells us that all reducible tasks are \learnableF, and Theorem~\ref{theo:final} allows us to detect when the final policy has been formed.
To illustrate these theorems, we now consider two examples of tasks that are reducible, in Section~\ref{sub:grid} and Section~\ref{sub:chain}, respectively. 
Our aim is not to show practical efficiency of the cycle-detection learning algorithm, but rather to illustrate the theoretical insights.
Indeed, because the considered examples are reducible, they are \learnableF\ by Theorem~\ref{theo:algvisit}. 
Next, aided by Theorem~\ref{theo:final}, we can experimentally measure how long it takes for the learning process to convergence.
In Section~\ref{sec:conclusion}, for further work, we identify aspects where the learning algorithm could be improved to become more suitable for practice.

\subsection{Grid Navigation Tasks}
\label{sub:grid}

Our first example is a navigation task in a grid world~\cite{sutton-barto_1998,potjans_2011}. 
In such a task, it is intuitive to imagine how paths are formed and what they mean.
Below we formalize a possible version of such a grid task.

The grid is represented along two axes, the $X$- and $Y$-axis.
At any time, the agent is inside only one grid cell $(x,y)$.
We let the state set $\states$ be a subset of $\nat\times\nat$.    
Let $\gridgoals\subseteq\states$ be a subset of cells, called \emph{goal cells}.
The agent could apply the following actions to each grid cell: 
    finish, 
    left, right, up, down, 
    left-up, left-down, right-up, and right-down.
Let $\gridactions$ denote the set containing these actions.
The finish action is a non-movement action that gives immediate reward when applied to a goal cell. 
The finish action intuitively says that the agent believes it has reached a goal cell and claims to be finished. Activating the finish action in any non-goal cell will just leave the agent in that cell without reward.
The actions other than finish will be referred to as movement actions.

For every movement action $\act$, there is noise from the environment. We formalize this with a noise offset function that maps each movement action to the possible relative movements that it can cause. For example, $\offsets{\text{left}}=\set{(-1,0), (-1,-1),\allowbreak (-1,1)}$, and $\offsets{\text{left-up}}=\set{(-1,1), (-1,0),\allowbreak (0,1)}$. 
Intuitively, noise adds one left-rotating option and one right-rotating option to the main intended direction. The offsets of the other movement actions can be similarly defined (see Appendix~\ref{app:examples}). 
For uniformity we define $\offsets{\text{finish}}=\set{(0,0)}$.

For a cell $(x,y)\in\nat\times\nat$ and an action $\act\in\gridactions$ we define the set $\move{(x,y)}\act$ of child-cells that result from the application of the offsets of $\act$ to $(x,y)$. Formally, we define
\[
    \move{(x,y)}\act = \set{(x + u,\, y + v) \mid (u,v)\in\offsets\act}.
\]

Now, we say that a task $\task=\tasktup$ is a \emph{grid navigation task} if
    $\states\subseteq\nat\times\nat$,
    $\initstates\subseteq\states$, 
    $\actions=\gridactions$,
    there exists some nonempty subset $\gridgoals\subseteq\states$ such that 
    \[
        \rewards=\set{(\st,\text{finish})\mid \st\in\gridgoals},
    \]
    and for each $(\st,\act)\in\states\times\actions$ we have
    \[
        \tr(\st,\act) = \begin{cases}
            \move\st\act & \text{if }\move\st\act\subseteq\states \\
            \set{\st} & \text{otherwise.}
            \end{cases}
    \]
Note that we only perform a movement action if the set of child-cells is fully contained in the set of possible grid cells; otherwise the agent remains stationary.%
    \footnote{This restriction results in policies that have sufficiently intuitive visualizations; see Figures~\ref{fig:box}, \ref{fig:corridor_template}, and \ref{fig:forward_backward}.}
    
The assumption of reducibility can additionally be imposed on grid navigation tasks.
In Figure~\ref{fig:box}, we visualize a grid navigation task that, for illustrative purposes, is only partially reducible.

\begin{figure}
    \begin{center}
    \includegraphics[width=0.5\textwidth]{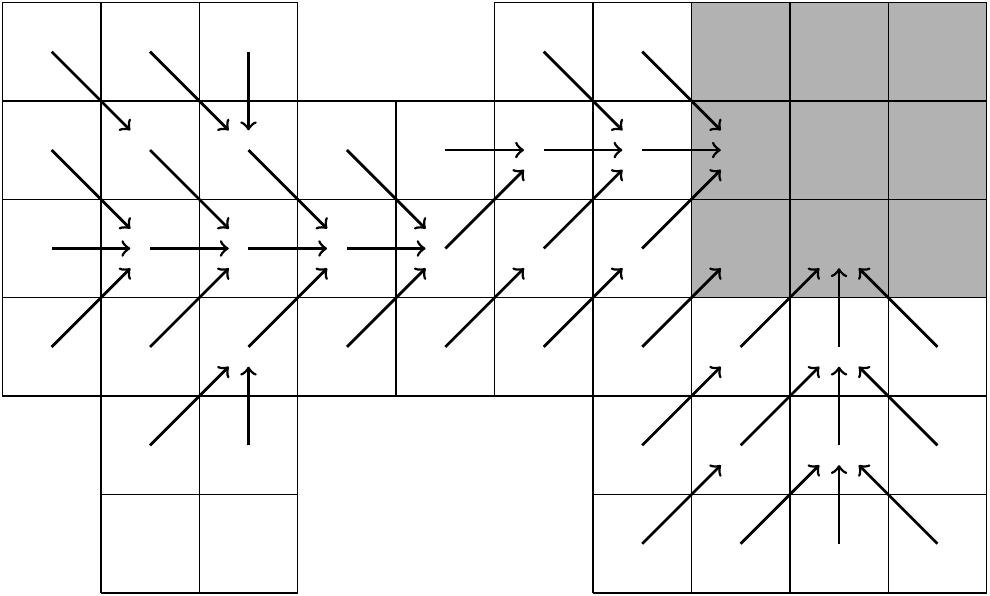}
    \end{center}
    \caption{
    A grid navigation task.    
    The shaded cells are goal cells, where the finish action should be executed to obtain reward.
    Reducibility of cells is illustrated by the arrows: an arrow leaving a cell represents an action that the cell could choose to come closer to a goal cell.
    We note that each arrow displays only the main intended direction of the action, and that in principle the noise offsets could occur during execution.
    In this example, not all cells are reducible. 
    For the non-reducible cells, for each applied movement action $\act$, either \romI\ function $\tr$ defines $\act$ to stay stationary for the cell, or \romII\ $\act$ could lead to another non-reducible cell.
    Naturally, if all non-reducible cells are removed then the resulting task is reducible.
    }
    \label{fig:box}
\end{figure}

We now discuss two simulation experiments that we have performed on such grid navigation tasks.
Some details of the experiments can be found in Appendix~\ref{app:examples}.
Let $\natzero$ denote the set of natural numbers without zero.

\paragraph{Convergence-Trial Index}
Our first experiment measures the convergence-trial index. First we discuss the general setup of the experiment.
We recall from Section~\ref{sub:final} that the convergence-trial index of a fair run is the first trial index where the final policy occurs at the end.
For a given task $\task$, we can simulate fair runs, and we stop each run when the final policy is detected through the characterization of Theorem~\ref{theo:final}; we remember the convergence-trial index. 
Each run is started with a random policy, where each state is assigned a random action by uniformly sampling the available actions. 
Fairness depends on the mechanism for choosing among successor states, which is also based on a uniform distribution (see Appendix~\ref{app:examples}). Interestingly, we do not have to simulate infinite runs because we always eventually detect the final policy; when we stop the simulation after finite time, the finite run may be thought of as a prefix of an infinite fair run.%
    \footnote{Of course, it is impossible to simulate an infinite run in practice.}

In principle, there is no upper bound on the convergence-trial index in the simulation. 
Fortunately, our experiments demonstrate that for the studied reducible tasks, there appears to be a number $i\in\natzero$ such that the large majority of simulated runs has a convergence-trial index below $i$. Possibly, there are outliers with a very large convergence-trial index, although such outliers are relatively few in number. 
We can exclude outliers from a list of numbers by considering a $p$-quantile with $0 < p < 1$.%
    \footnote{In Appendix~\ref{app:examples}, we give the precise definition of $p$-quantile that we have used in our analysis.}

We wanted to see if the convergence-trial index depends on the distance between the start cells and the goal cells. For this purpose, we have considered grid tasks with the form of a corridor, as shown in Figure~\ref{fig:corridor_template}. There is only one start cell. The parameter that we can change is the distance from the start cell to the patch of goal cells. All other aspects remain fixed, including the width of the corridor and the number and the location of the goal cells. The arrows in Figure~\ref{fig:corridor_template} show reducibility of this kind of task.
For a number $l\in\natzero$, we define the $l$-corridor as the corridor navigation task where the distance between the start cell and the goal cells is equal to $l$.

Now, for some lengths $l$, we have simulated runs on the $l$-corridor.
For each $l$-corridor separately, we have simulated $400$ runs and we have computed the $0.9$-quantile for the measured convergence-trial indexes; this gives an empirical upper bound on the convergence-trial index for the $l$-corridor.
Figure~\ref{fig:grid_convergence} shows the quantiles plotted against the corridor lengths.
We observe that longer corridors require more time to learn. This is probably because at each application of a movement action to a cell, there could be multiple successor cells due to nondeterminism. Intuitively, the nondeterminism causes a drift away from any straight path to the goal cells.
The policy has to learn a suitable action for each of the cells encountered due to drift.
So, when the corridor becomes longer, more cells are encountered due to nondeterminism, and therefore the learning process takes more time to learn a suitable action for each of the encountered cells.
Figure~\ref{fig:grid_convergence} suggests an almost linear relationship between corridor length and the empirical upper bound on the convergence-trial index based on the 0.9-quantile. In Section~\ref{sub:chain}, we will see another example, where the relationship is not linear.

\begin{figure}
    \begin{center}
    \includegraphics[width=0.4\textwidth]{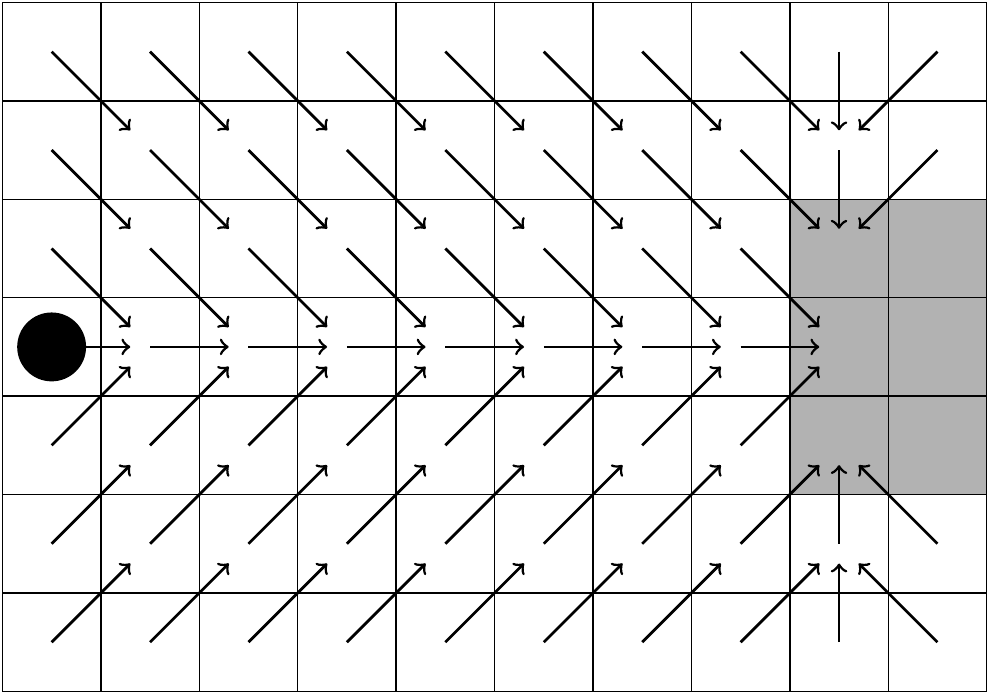}
    \end{center}
    \caption{A corridor template for grid navigation tasks. The single start cell is marked with a black circle. The goal cells are shaded in gray. The parameterizable part is the horizontal distance between the start cell and the patch of goal cells. For completeness, the arrows represent a possible policy that illustrates reducibility of this kind of task. Although not shown, we recall that the finish action should be executed on the goal cells. The length of the depicted corridor is $7$, which is the distance between the start cell and the patch of goal cells.}
    \label{fig:corridor_template}
\end{figure}

\begin{figure}
\begin{center}
\begin{subfigure}[t]{.45\textwidth}
    \includegraphics[width=1\textwidth]{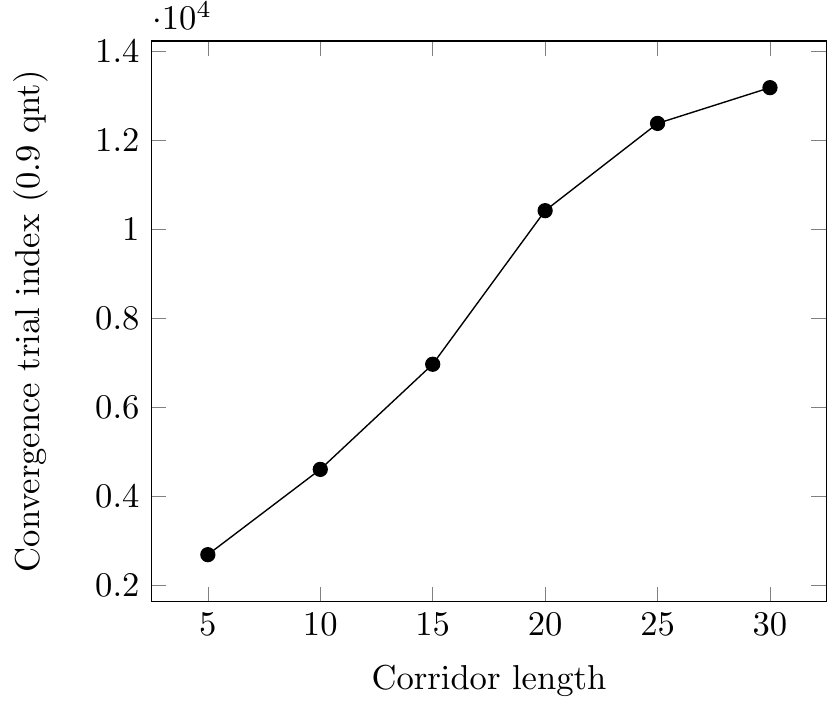}
    \subcaption{Convergence-trial index related to corridor length.}
    \label{fig:grid_convergence}    
\end{subfigure}
\quad
\begin{subfigure}[t]{.45\textwidth}
    \includegraphics[width=1\textwidth]{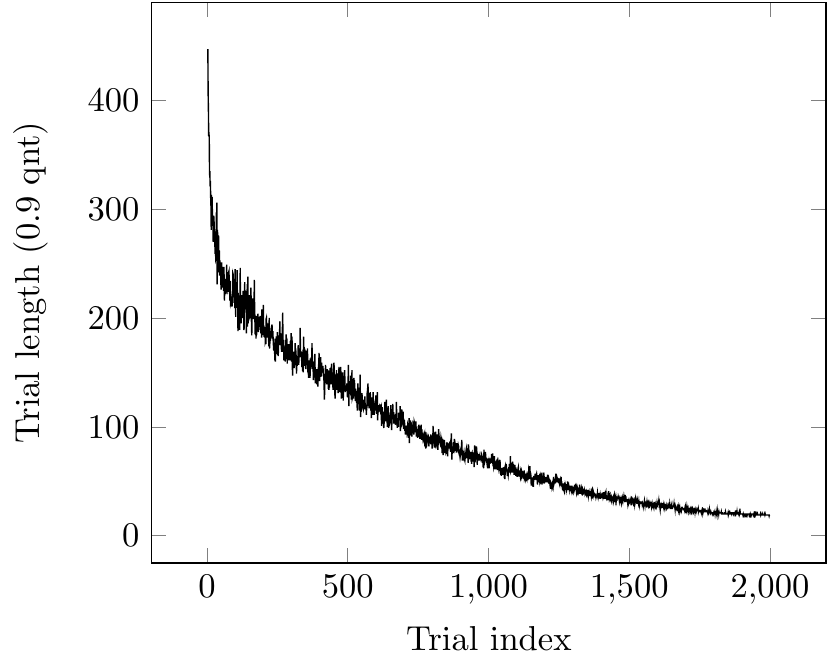}
    \caption{Trial length related to trial index for the corridor task with length $10$. We have excluded the first two trials to better show the shape of the curve.}
    \label{fig:grid_trial_length}
\end{subfigure}
\end{center}
\caption{Plots for the corridor grid navigation tasks.}
\end{figure}

\paragraph{Trial Length}
For a fixed grid navigation task, we also wanted to test if trial length decreases as the run progresses. 
A decreasing trial length would demonstrate that the learning algorithm is gradually refining the policy to go more directly to reward from the start states, avoiding cycles in the state space.

For the fixed corridor length of $l=10$, we have simulated the first 2000 trials of 1000 runs; and, for each trial we have measured its length as the number of transitions. 
This gives a data matrix where each cell $(i,j)$, with run index $i$ and trial index $j$, contains the length of trial $j$ in the simulated run $i$. 
For each trial index we have computed the $0.9$-quantile of its length measurements in all runs. By plotting the resulting empirical upper bound on trial length against the trial index, we arrive at Figure~\ref{fig:grid_trial_length}.
We can see that trial length generally decreases as the run progresses.
A similar experimental result is also reported by \citet{potjans_2011}, in the context of neurons learning a grid navigation task.

\paragraph{Visualizing Forward and Backward}

We recall from Theorem~\ref{theo:final} that, within the context of a specific task, the final policy $\pol$ in a converging run satisfies the inclusion $\forward\subseteq\backward$.
For the corridor in Figure~\ref{fig:corridor_template}, for one simulated run, we visualize the forward state set and the backward state set of the final policy in Figure~\ref{fig:forward_backward}. 
Note that the forward state set is indeed included in the backward state set.
Interestingly, in this case, the simulated run has initially learned to perform the finish action in some goal cells, as witnessed by the backward state set, but eventually some of those goal cells are no longer reached from the start state, as witnessed by the forward state set.

\begin{figure}
\begin{center}
\begin{subfigure}[t]{.45\textwidth}    
    \includegraphics[width=1\textwidth]{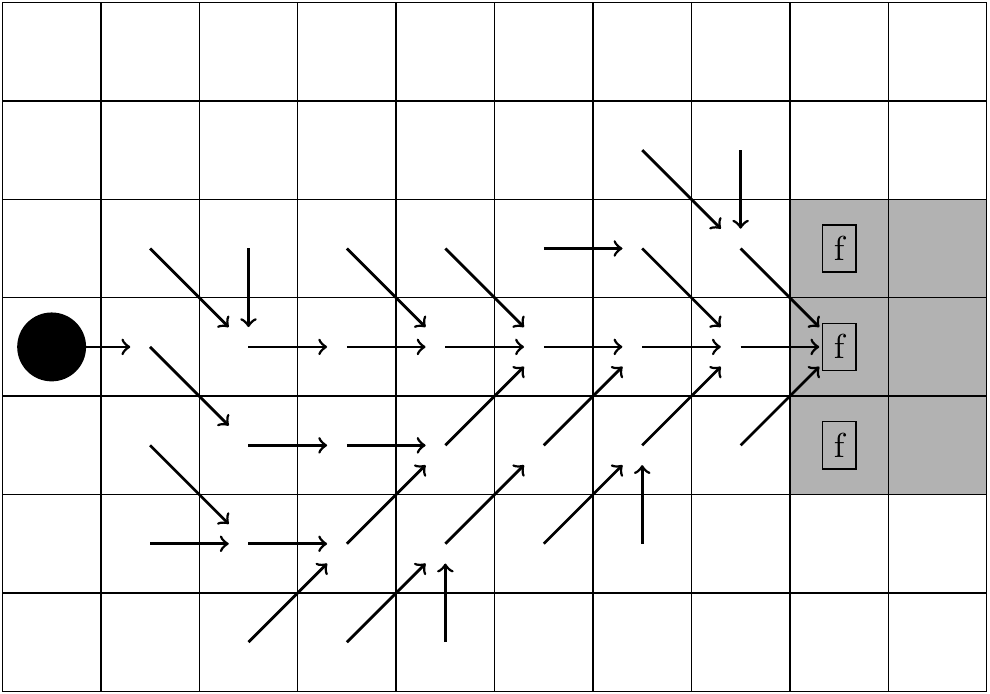}
    \subcaption{Forward state set.}    
\end{subfigure}
\quad
\begin{subfigure}[t]{.45\textwidth}
    \includegraphics[width=1\textwidth]{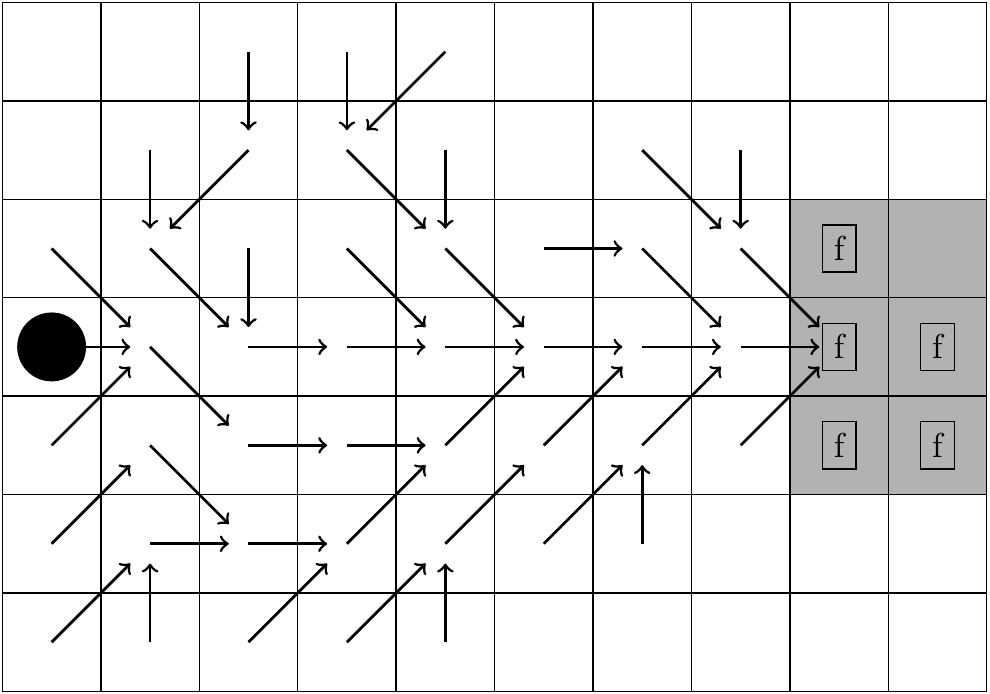}
    \caption{Backward state set.}    
\end{subfigure}
\end{center}
\caption{Visualization of the forward state set and backward state set, from Section~\ref{sub:final}, of the final policy in one simulated run of the corridor navigation task shown in Figure~\ref{fig:corridor_template}. %
The states belonging to a set are marked with their action in the final policy; the empty grid cells are not part of the set. Each arrow shows only the main intended direction of the action; we do not show the left-rotating noise option and the right-rotating noise option. %
Goal cells that are assigned the finish action by the final policy are marked with the ``f'' sign.}
\label{fig:forward_backward}
\end{figure}

\subsection{Chain Tasks}
\label{sub:chain}

In addition to the concrete grid tasks of Section~\ref{sub:grid}, we have also considered a slightly more abstract form of task, that we call chain task.
The general form of a chain task is shown in Figure~\ref{fig:chain}.
The parameter $n$ tells us how long the chain is; the states of the chain are, in order, $1,\ldots,n$ and one final state $n+1$.
To obtain reward from the start state, we should in the worst case perform all actions $a_1,\ldots,a_n$, in sequence, finished by one arbitrary action. For each $i\in\set{1,\ldots,n}$, the action $a_i$ should be applied to state $i$.
But there is forward nondeterminism that, for each pair $(i,a_i)$ could send us to an arbitrary state later in the chain, closer to state $n+1$.
Also, there are backward deterministic transitions that take us back to the start state whenever we apply the wrong action to a state.
Formally, for a fixed value of $n\in\natzero$, we obtain a graph $\task_n=\tasktup$, defined as follows: 
    $\states=\set{1,\ldots,n,n+1}$, 
    $\initstates=\set{1}$,
    $\actions=\set{a_1,\ldots,a_n}$,
    $\rewards=\set{(n+1,a)\mid a\in\actions}$, 
    and regarding $\tr$, we define
    \begin{itemize}
        \item for each $i\in\states$, 
            \begin{align*}
                \tr(i,a_i) &= \set{i+1,\ldots,n+1},\\        \tr(i,b) &= \set{1}\text{ for each }b\in\actions\setminus\set{a_i}\text{; and,}
            \end{align*}
        \item for the goal state $n+1$, we define $\tr(n+1,a)=\set{n+1}$ for each $a\in\actions$.
    \end{itemize}
Note that for each $n\in\natzero$, the task $\task_n$ is reducible: conceptually, the reducibility iterations first assign action $a_n$ to state $n$, then action $a_{n-1}$ to state $n-1$, and so on, until state $1$ is assigned action $a_1$.

\begin{figure}
    \begin{center}
    \includegraphics[width=0.9\textwidth]{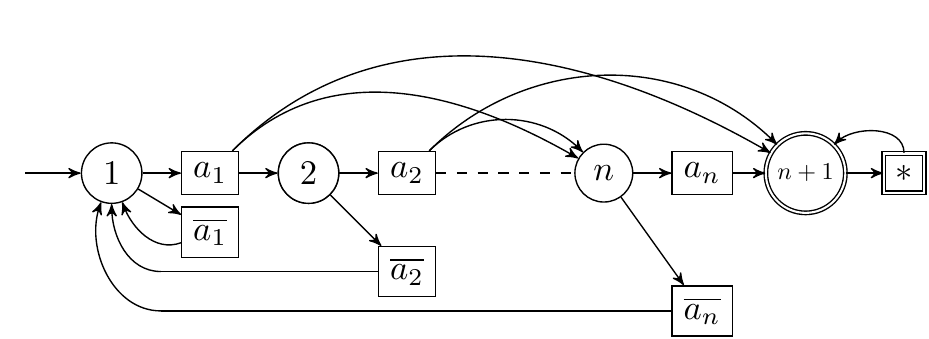}
    \end{center}
    \caption{Template for chain tasks. Regarding notation, the symbol $*$ denotes the set of all actions $\set{a_1,\ldots,a_n}$, and the symbol $\overline{a_i}$ with $i\in\set{1,\ldots,n}$ stands for the set of all actions except action $a_i$. The graphical notation of tasks is explained in Figure~\ref{fig:reduce}.}
    \label{fig:chain}
\end{figure}

On chain tasks, we have performed simulation experiments that are similar to the experiments on the grid navigation tasks in Section~\ref{sub:grid}. These experiments are discussed next.

\paragraph{Convergence-trial Index}
Using the same experimental procedure as for grid tasks, we have simulated runs for some chain lengths. For each chain length separately, we have simulated 400 runs, and we have computed the $0.9$-quantile of the measured convergence-trial indexes.
By plotting the resulting empirical upper bound on the convergence-trial index against the chain length, we arrive at Figure~\ref{fig:chain_convergence}.
We see that the convergence-trial index rises faster than linear in terms of the chain length; this is in contrast to Figure~\ref{fig:grid_convergence} for grid corridors.
One possible explanation, is that the forward nondeterminism on the chain causes each state to be visited less frequently in a simulated run. 
For a fixed length $n$, the effect is that a state $\st\in\set{2,\ldots,n}$ could stay connected longer to a bad action, leading back to the start state $1$. Of course, at the end of each trial, the start state should be connected to action $a_1$, because otherwise no progress could have been made; but the other states could in principle have any action.
So, we might not yet encounter the final policy for a long time, as recognizable through the syntactic characterization of Theorem~\ref{theo:final}.

\paragraph{Trial Length}
Using the same experimental procedure as for grid tasks, for the fixed chain length of $n=10$, we have simulated the first 2000 trials of 1000 runs, and we have computed the $0.9$-quantile on trial length as explained for the grid corridor experiment. 
By plotting the resulting empirical upper bound on trial length against the trial index, we arrive at Figure~\ref{fig:chain_trial_length}. Again, this figure suggests that the learning algorithm is able to gradually improve the policy over trials.

\paragraph{Form of the Final Policy}
We describe the form of the final policy, rather than visualizing it, because the form is very restricted.
Consider the final policy $\pol$ in a simulated run of the chain task with length $n\in\natzero$. 
In the simulation, we know that each trial should have ended with action $a_1$ assigned to state $1$ because otherwise the state $n+1$ could not have been reached.
This property also applies to the convergence-trial, so $\pol(1)=a_1$.
Hence, $\forward$ is the set of all states due to the forward nondeterminism along the chain. 
We recall from Theorem~\ref{theo:final} that the final policy satisfies $\forward\subseteq\backward$.
Therefore, $\backward$ is also the set of all states.
We can now see that the final policy $\pol$ has to satisfy $\pol(i)=a_i$ for all $i\in\set{1,\ldots,n}$. 
Towards a contradiction, let $i$ be the largest state number for which $\pol(i)\neq a_i$. Then, using the iterations for computing $\backward$, we can see that $\set{i+1,\ldots,n+1}\subseteq\backward$ but $\set{1,\ldots,i}\cap\backward=\emptyset$.%
    \footnote{Here, the computation of $\backward$ is started with the set $\ground=\set{n+1}$.}
    This would be the desired contradiction, because $\backward$ is the set of all states.

\begin{figure}
\begin{center}
\begin{subfigure}[t]{0.45\textwidth}    
    \includegraphics[width=1\textwidth]{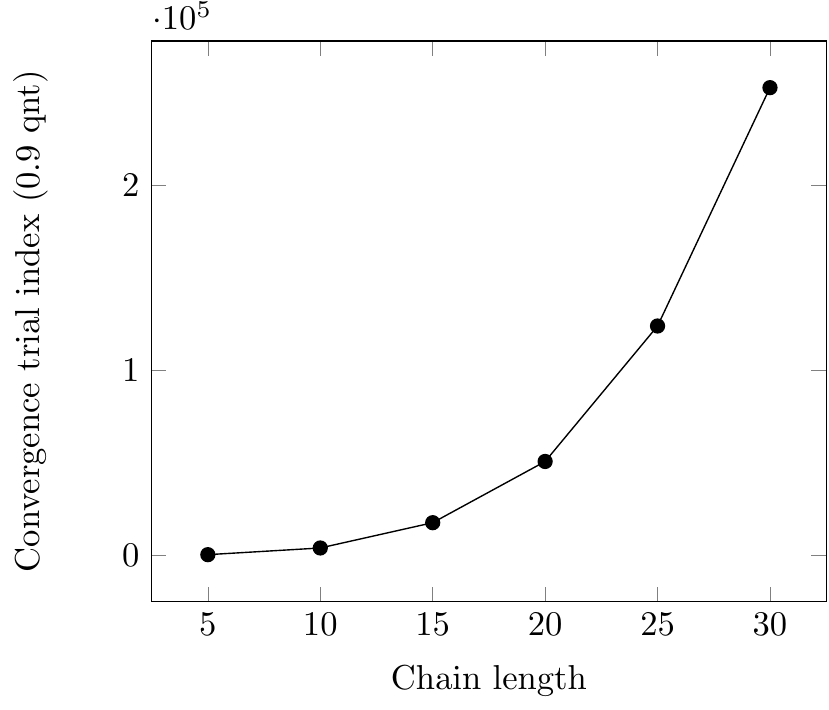}    
    \caption{Convergence-trial index related to chain length.}
    \label{fig:chain_convergence}
\end{subfigure}
\quad
\begin{subfigure}[t]{.45\textwidth}    
    \includegraphics[width=1\textwidth]{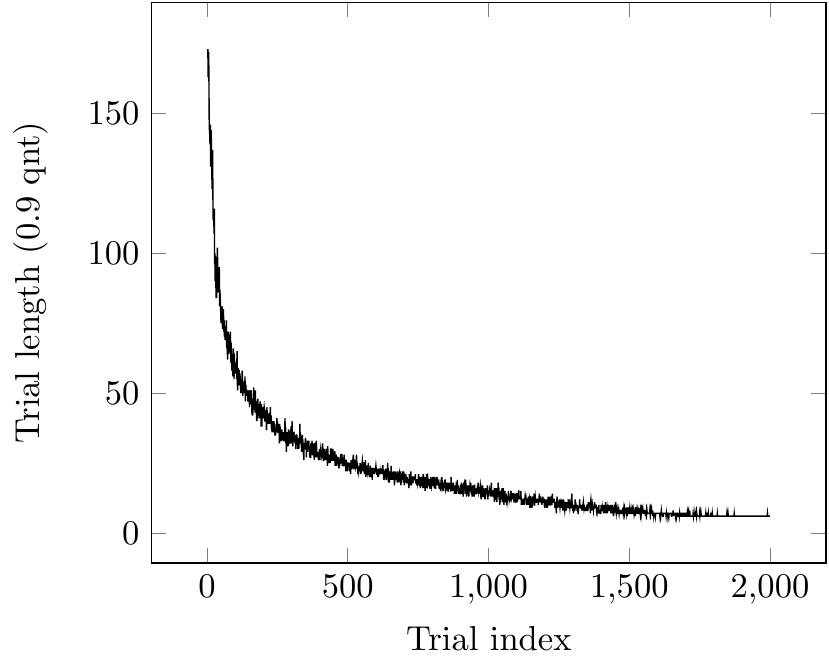}    
    \caption{Trial length related to trial index for the chain task with length $n=10$. We have excluded the first two trials to better show the shape of the curve.}
    \label{fig:chain_trial_length}
\end{subfigure}
\end{center}
\caption{Plots for the chain tasks.}
\end{figure}

\section{Conclusion and Further Work}
\label{sec:conclusion}

We have studied the fascinating idea of reinforcement learning in a non-numeric framework, where the focus lies on the interaction between the graph structure of the task and a learning algorithm.
We have studied the graph property of reducibility, that implies the existence of a policy that makes steady progress towards reward despite nondeterminism in the task. Interestingly, reducibility, combined with a natural fairness assumption, enables our simple learning algorithm to learn the task.
We have also characterized the final policy for converging runs, which allows the precise detection of the convergence-trial in simulations.
We now discuss some avenues for further work.

\paragraph{Characterizing Learnable Tasks}
We have seen a sufficient property (Theorem~\ref{theo:algvisit}) and necessary properties (Proposition~\ref{prop:necessary}) for tasks to be \learnableF. 
The gap between the sufficient and necessary properties seems strongly related to the unstable set of a task (see Section~\ref{sub:necessary}).
Perhaps it is possible to characterize the tasks that are \learnableF\ by imposing some additional constraints on the manner by which the unstable set is connected to the reducible states, in addition to the already identified necessary properties.

\paragraph{Time Before Convergence}
Related to Remark~\ref{remark:theo-algvisit} and to the simulations in Section~\ref{sec:examples}, one could try to theoretically provide an upper bound on the convergence-trial index, for some class of tasks. 
We could also make assumptions regarding the probability distributions underlying the random actions proposed for a state and the choice of successor states when applying an action. 
For a given task, the result could be a probability distribution on the convergence-trial index,  or on the total number of transitions before convergence.

\paragraph{Fading Eligibility Traces}
In some models of biologically plausible learning, the activation between any pair of connected neurons is represented by an eligibility trace~\cite{sutton-barto_1998,fremaux_2013,gerstner_book_2014}. 
When obtaining reward, the value of this trace is applied to the synaptic weight between the neurons. In realistic scenarios, the traces fade, with the advantage that a simulation or practical setup does not have to keep remembering all the information of past states before obtaining reward. 
The current article may be viewed as studying eligibility traces that are non-fading, because the working memory of the cycle-detection learning algorithm is effectively stored until the end of the trial.
It appears interesting to let the working memory fade, perhaps by modeling the working memory as a first-in-first-out queue, where a newly entering state would remove the oldest state from the queue once a size limit has been reached.

\paragraph{Negative Feedback}
In this article, any path through the state space from a given start state to reward is good. However, in some applications we want to avoid certain states. For example, in a navigation task, an organism might want to reach its nest from a certain starting location, but on the way to the nest the organism should avoid hazardous locations like pits or swamps. 
It appears interesting to formally investigate such cases.
Information about hazards could be incorporated by extending the framework of this article with an explicit set $\mathit{hazards}$ in tasks, which contains the state-action pairs which should be avoided; this is the opposite of the set $\rewards$.
The framework could further be extended to differentiate between trials that terminate with reward and trials that terminate with a hazard.
A run could be said to convergence if eventually all trials terminate with reward, and all states eventually become stable.

\paragraph{Incomplete Information and Generalization}
The framework studied in this article provides complete information to the learning algorithm because each individual state can be mapped to its own action.
In real-world applications, such as robot navigation~\cite{thrun_book_2005}, the agent can only work with limited sensory information available in each time step. In that case, the agent should first build concepts for states in order to differentiate them. These concepts should be made by remembering sensory information over time, and in general multiple states will remain grouped together under the same concept because sensory information is not sufficiently accurate to differentiate between all states.
This issue was also raised as an item for future work by \citet{fremaux_2013}, who initially also have considered a framework in which complete information is available to the learning agent.
There is ongoing work on partially observable tasks, see e.g.~\cite{littman_1995,chatterjee_2015}.

Building concepts is related to the problem of generalization~\cite{sutton-barto_1998} because in real-world tasks there might be too many states to store in the policy. It would be useful to collect states in conceptual groups and then assign an action to each group.

It appears interesting to formalize incomplete information and generalization in an extended version of the current framework, and to investigate sufficient and necessary properties for convergence.

\bibliographystyle{apalike} 

\newpage
\appendix

\section*{Appendix}

\addtocontents{toc}{\protect\setcounter{tocdepth}{0}}

\section{Examples}
\label{app:examples}

This Section contains additional details for the example tasks discussed in Section~\ref{sec:examples}.

\subsection{Grid Action Offsets}
For our formalization of grid tasks, Table~\ref{table:grid-offsets} gives the offsets for each action.

\begin{table}
    \caption{Grid action offsets. For each movement action, we first give the main intended direction, followed by a left-rotating noise option and a right-rotating noise option.}
    \label{table:grid-offsets}
    \begin{center}
    \begin{tabular}{|l|l|}
        \hline
        Action & Possible offsets\\
        \hline
        \hline
        finish & (0,0)\\
        \hline
        left & (-1,0), (-1,-1), (-1,1) \\
        \hline
        right & (1,0), (1,-1), (1,1) \\
        \hline
        up & (0,1), (-1,1), (1,1) \\
        \hline
        down &(0,-1), (1,-1), (-1,-1) \\
        \hline
        left-up & (-1,1), (-1,0), (0,1) \\
        \hline
        left-down & (-1,-1), (0,-1), (-1,0)\\
        \hline
        right-up & (1,1), (0,1), (1,0) \\
        \hline
        right-down & (1,-1), (1,0), (0,-1) \\
        \hline
    \end{tabular}
    \end{center}    
\end{table}

\subsection{Quantiles}
We compute the quantiles with the statistics package R.
For completeness, we recall here the definition of quantiles that we have used in our analysis; this definition is called Definition 1 by \citet{hyndman_1996}. 

Let $L$ be a nonempty list of numbers, possibly containing duplicates. Let $\len L$ denote the length of $L$. For an index $i\in\set{1,\ldots,\len L}$, we write $\elemat Li$ to denote the number at index $i$ in $L$.
We write $\order L$ to denote the ordered version of $L$, where the numbers are sorted in ascending fashion; we keep duplicates, so $\len{\order L}=\len L$.

Let $p\in\real$ with $0 < p < 1$. 
The $p$-quantile of a nonempty list $L$ of numbers, denoted $\qnt pL$ is defined as follows: denoting $j=\floor{p\cdot n}$ where $n=\len L$,
\[
    \qnt pL = 
        \begin{cases}
           \elemat{\order L}{j+1} & \text{if }p\cdot n - j > 0\\
           \elemat{\order L}{j} & \text{if }p\cdot n - j = 0.
        \end{cases}
\]

Intuitively, using the index $j=\floor{p\cdot n}$ in the ordered list $\order L$ is a good attempt at finding a number $v$ such that a fraction $p$ of $L$ is smaller than or equal to $v$. 
The assumptions $0 < p$ and $p < 1$ ensure that we only apply valid indexes in the range $\set{1,\ldots,\len L}$ to the list $\order L$.

\subsection{Implementation Notes}
The simulation was written with Java Development Kit 8.
In our experimental results, we have measured only the number of transitions and the number trials in runs. Since no wall-clock time was needed, the exact running time of the simulation in seconds was not measured.
During every transition, we have used a uniform sampling from the possible successor states. 
Concretely, given an array of successor states, we have used the function \java{Math.random()} to generate random indexes in this array, as follows:
    the double precision number returned by \java{Math.random()} can be converted into an integer by multiplying with the array length and subsequently truncating the resulting number with the \java{Math.floor()} function.

\end{document}